\documentclass{article}

\PassOptionsToPackage{numbers, compress, sort}{natbib}
\usepackage[preprint]{neurips_2026}

\usepackage[utf8]{inputenc}
\usepackage[T1]{fontenc}
\usepackage[hidelinks]{hyperref}
\usepackage{url}
\usepackage{booktabs}
\usepackage{amsfonts}
\usepackage{amsmath}
\usepackage{amssymb}
\usepackage{amsthm}
\usepackage{nicefrac}
\usepackage{microtype}
\usepackage[dvipsnames,table]{xcolor}
\usepackage{graphicx}
\usepackage{float}
\usepackage{algorithm}
\usepackage{algpseudocode}
\usepackage{multirow}
\usepackage{siunitx}
\usepackage{enumitem}
\usepackage[skins]{tcolorbox}
\usepackage{wrapfig}
\usepackage{makecell}
\usepackage{subcaption}
\usepackage{arydshln}
\usepackage{placeins}
\usepackage{tabularx}

\newcommand{\widegraphics}[2][0.88\linewidth]{%
  \includegraphics[width=#1]{#2}%
}
\newcommand{\widecaption}[1]{\captionsetup{width=0.98\linewidth}\caption{#1}}
\newcommand{\widetablebox}[1]{\makebox[\linewidth][c]{#1}}
\AtBeginDocument{%
  \newgeometry{%
    textheight=9.05in,
    textwidth=6.35in,
    top=0.85in,
    headheight=12pt,
    headsep=18pt,
    footskip=30pt
  }%
  \raggedbottom
}

\newtheorem{proposition}{Proposition}

\newcommand{\clip}{\operatorname{clip}}

\newcommand{\pit}[1]{\pi_\theta(\cdot\mid x,o_{<t}^{(#1)})}
\newcommand{\pitr}[1]{\pi_\theta(\cdot\mid x,r,o_{<t}^{(#1)})}

\definecolor{tchcol}{RGB}{30,80,170}
\definecolor{prvcol}{RGB}{178,34,34}
\definecolor{bsdcol}{RGB}{34,139,34}

\newcommand{\prv}[1]{\textcolor{prvcol}{#1}}

\title{Tailoring Teaching to Aptitude: Direction-Adaptive Self-Distillation for LLM Reasoning}

\author{
Hongbin Zhang\textsuperscript{1,2} \quad
Chaozheng Wang\textsuperscript{3} \quad
Kehai Chen\textsuperscript{1,2} \quad
Youcheng Pan\textsuperscript{2} \\
\textbf{Yang Xiang}\textsuperscript{2} \quad
\textbf{Jinpeng Wang}\textsuperscript{3} \quad
\textbf{Min Zhang}\textsuperscript{1,2} \\
\\
\textsuperscript{1}Institute of Computing and Intelligence, Harbin Institute of Technology, Shenzhen, China \\
\textsuperscript{2}Peng Cheng Laboratory, Shenzhen, China \\
\textsuperscript{3}Keeta AI, Meituan, Beijing, China
}

\begin{document}

\maketitle

\begin{abstract}
On-policy self-distillation (OPSD) is an emerging LLM post-training paradigm in which the model serves as its own teacher: conditioned on privileged information such as a reference trace or hint, the same policy provides dense token-level supervision on its own rollouts.
However, recent studies show that OPSD degrades complex reasoning by suppressing predictive uncertainty, which supports exploration and hypothesis revision.
Our token-level analysis shows that this failure arises from applying a uniform direction of teacher supervision across tokens with different uncertainty levels: conformity to the privileged self-teacher suppresses exploration at high entropy, while deviation from the teacher degrades step accuracy at low entropy.
Accordingly, we propose \textbf{Direction-Adaptive Self-Distillation} (\textbf{DASD}), which reframes privileged self-distillation from uniform teacher imitation into entropy-routed directional supervision: high-entropy tokens are pushed away from the privileged teacher to preserve exploration, while low-entropy tokens are pulled toward the teacher to stabilize step-level execution.
Across six mathematical reasoning benchmarks, DASD achieves the best macro Avg@16 over strong RLVR and self-distillation baselines. Pass@$k$, reasoning-health, and generalization analyses
show that these average gains come from preserving exploration without sacrificing step-level execution.
\end{abstract}

\section{Introduction}
Knowledge distillation~\citep{hinton2015distilling,kim2016sequence,furlanello2018born,xu2024kdsurvey} and its self-distilling variants~\citep{snell2022learning,song2026survey} have recently emerged as efficient ways to improve reasoning in large language models (LLMs) during post-training with reinforcement learning from verifiable rewards (RLVR)~\citep{schulman2017ppo,christiano2017deeprl,ouyang2022instructgpt,shao2024deepseekmath,guo2025deepseekr1,yu2026dapo,zhang2025rlsurvey,liu2025rlllmsurvey,chen2025cotsurvey}. In on-policy self-distillation (OPSD), the model generates its own rollouts as a student, while a privileged conditioning of the same model---given information such as a reference trace, solution hint, or feedback---serves as a teacher that provides dense token-level guidance~\citep{zhao2026opsd,li2026srpo,hubotter2026sdpo,yang2026rlsd,shenfeld2026selfdistillation}. OPSD is attractive because it couples RLVR's outcome supervision with fine-grained teacher feedback, without requiring a separate external teacher~\citep{shao2024deepseekmath,zhao2026opsd,hubotter2026sdpo}.

Despite this appeal, prior work shows that OPSD can degrade reasoning on complex problems by suppressing uncertainty expressions during reasoning~\citep{kim2026whyopsd}. Privileged conditioning on reference information $r$ makes the self-teacher concentrate on a confident, solution-conditioned path; imitation then pushes the student toward the same low-uncertainty style. This suppresses markers such as ``wait'', which help an unprivileged solver explore alternatives, revise hypotheses, and recover from mistakes, thereby weakening the exploratory search and self-correction needed for difficult reasoning.

This diagnosis motivates a natural question: if conformity to the privileged teacher harms reasoning by suppressing uncertainty, can reversing the distillation direction restore it? Our sign-flip probe shows that it cannot. Deviation recovers epistemic markers, but corrupts step-level execution. This complementary failure indicates that teacher pressure is not uniformly beneficial or harmful; its effect depends on the token regime. Prefix interventions further isolate this mechanism: conformity damages high-entropy forking tokens, where exploration must be preserved, whereas deviation damages low-entropy scaffolding tokens, where stable execution must be maintained.

\begin{figure}[H]
  \centering
  \widegraphics[0.80\linewidth]{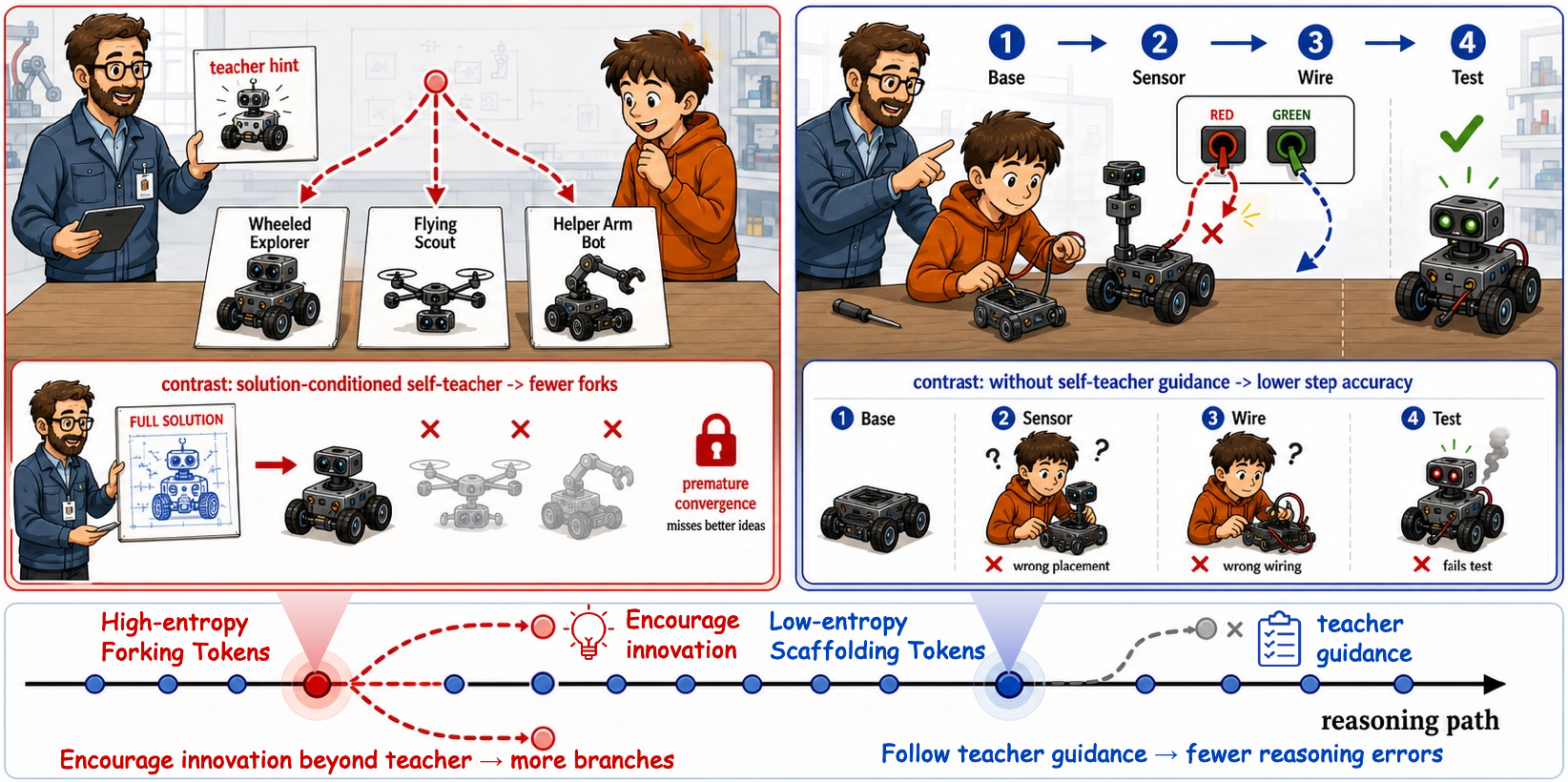}
  \widecaption{\textbf{Conceptual analogy for DASD.} Student entropy switches the role of the solution-conditioned self-teacher. High-entropy forking tokens should move away from the teacher to avoid premature convergence and preserve alternative reasoning paths, whereas low-entropy scaffolding tokens should follow the teacher to prevent routine execution errors.}
  \label{fig:concept-analogy}
\end{figure}

Building on this dissection, we turn the analysis into a concrete training prescription: \textbf{Direction-Adaptive Self-Distillation} (\textbf{DASD}).
As illustrated in Figure~\ref{fig:concept-analogy}, high-entropy forking tokens correspond to decision points where directly imitating the solution-conditioned teacher may prematurely collapse alternative reasoning paths. DASD therefore pushes these tokens away from the teacher to preserve exploration. By contrast, low-entropy scaffolding tokens correspond to routine execution steps, where teacher guidance can help prevent local errors. DASD therefore pulls these tokens toward the teacher to stabilize reasoning.
In this way, DASD reframes privileged self-distillation from uniform teacher imitation into entropy-routed directional supervision: the verifier remains the global outcome anchor, while the self-teacher provides a local token-level signal whose direction is adapted according to entropy.

Across six mathematical reasoning benchmarks, it attains the best macro Avg@16 over strong RLVR and self-distillation baselines at every scale, with the largest average gains on harder competition-style benchmarks (HMMT25, OlympiadBench, AIME24, AIME25) that require maintaining multiple plausible solution paths. Pass@$k$ scaling and reasoning-health diagnostics show why: DASD expands exploration while preserving step-level execution, matching the token-level mechanism above.

Our contributions are threefold: (i) we identify a token-level directional mismatch behind OPSD degradation, where teacher conformity suppresses high-entropy forking tokens and teacher deviation corrupts low-entropy scaffolding tokens; (ii) we propose Direction-Adaptive Self-Distillation (DASD), an entropy-routed objective that pushes high-entropy tokens away from the privileged self-teacher and pulls low-entropy tokens toward it; (iii) and we show across six mathematical reasoning benchmarks that DASD improves performance by preserving exploration while maintaining step-level execution.

\section{Preliminaries}\label{sec:prelim}

\noindent\textbf{Group Relative Policy Optimization.}
We consider the standard RLVR setting in which a policy $\pi_\theta$ generates a response $o=(o_1,\dots,o_T)$ for a question $x$, and a programmatic verifier returns a binary reward $R(x,o)\in\{0,1\}$~\citep{shao2024deepseekmath,guo2025deepseekr1}. Group Relative Policy Optimization (GRPO)~\citep{shao2024deepseekmath} samples $G$ trajectories from $\pi_{\theta_{\mathrm{old}}}(\cdot\mid x)$, assigns $A_G^{(i)}=(R(x,o^{(i)})-\mu_G)/\sigma_G$ using the group's reward mean $\mu_G$ and standard deviation $\sigma_G$, and optimizes the clipped surrogate below, where $\rho_t^{(i)}=\pi_\theta(o_t^{(i)}\mid x,o_{<t}^{(i)})/\pi_{\theta_{\mathrm{old}}}(o_t^{(i)}\mid x,o_{<t}^{(i)})$:
\begin{equation}
\mathcal{L}_{\mathrm{GRPO}}(\theta)
= \mathbb{E}\!\left[
\frac{1}{G}\sum_{i=1}^{G}\frac{1}{|o^{(i)}|}
\sum_{t=1}^{|o^{(i)}|}
\min\!\left(
\rho_t^{(i)} A_G^{(i)},
\clip(\rho_t^{(i)},1-\epsilon,1+\epsilon)A_G^{(i)}
\right)
\right].
\label{eq:grpo}
\end{equation}
GRPO's limitation is \emph{sparse credit assignment}: every token in a trajectory inherits the same scalar $A_G^{(i)}$, so decisive deductions, routine algebra, and stylistic filler are pushed in the same direction.

\noindent\textbf{On-policy distillation and self-distillation.}
On-policy distillation (OPD)~\citep{agarwal2024onpolicy} densifies supervision by matching the student's next-token distribution to a teacher distribution on the student's own rollouts. OPD normally requires a stronger external teacher; on-policy \emph{self}-distillation (OPSD)~\citep{zhao2026opsd,hubotter2026sdpo} instead uses the same policy as a teacher, specialized by \emph{\prv{privileged information $r$}} such as a verified trace, solution hint, or environment feedback:
\begin{equation}
\mathcal{L}_{\mathrm{OPSD}}(\theta)
= \mathbb{E}_{o\sim\pi_\theta(\cdot\mid x)}\!\left[
\frac{1}{|o|}\sum_{t=1}^{|o|}
\mathrm{KL}\!\left(
\pi_\theta(\cdot\mid x,o_{<t})
\,\middle\|\,
\mathrm{sg}\!\left[\prv{\pi_\theta(\cdot\mid x,r,o_{<t})}\right]
\right)
\right].
\label{eq:opsd-rkl}
\end{equation}
The reverse-KL geometry is mode-seeking: it concentrates the student on teacher-preferred modes, thereby providing dense token-level credit without an external model.

\noindent\textbf{Why OPSD can fail on reasoning.}
Despite this efficiency, \citet{kim2026whyopsd} show that OPSD can degrade mathematical reasoning. Privileged conditioning resolves the teacher's residual uncertainty, so its outputs contain fewer epistemic markers such as ``wait'', ``hmm'', and ``alternatively''---tokens that help an unprivileged solver branch, revise, and recover. The student can therefore distill a \emph{solution-conditioned style}: it learns to generate with confidence that depends on $r$, even though $r$ is unavailable at inference time. This train--inference mismatch motivates the next section.

\section{A Token-Level Anatomy of OPSD's Reasoning Failure}\label{sec:prelim-analysis}

\refstepcounter{subsection}\label{sec:rq1}\noindent\textbf{Probe 1: Does reversing self-distillation direction restore reasoning?}
\noindent If imitating a privileged teacher harms reasoning by erasing uncertainty, a tempting repair is to reverse the direction and push the student away from that teacher. We treat this sign flip as a diagnostic probe rather than a candidate method. To isolate teacher-pressure direction from full-distribution privileged-information leakage~\citep{yang2026rlsd}, we use the sampled log-evidence gap:
$A_t \;=\; \prv{\log\pi_\theta(o_t\mid x, r, o_{<t})} \;-\; \log\pi_\theta(o_t\mid x, o_{<t}),$
and optimize $\max_\theta \mathbb{E}_{o\sim\pi_\theta}\!\big[\sum_t s_t A_t\big]$ with a global sign $s_t\in\{+1,-1\}$. \textbf{Conformity} ($s_t\!\equiv\!+1$) is the OPSD-like direction: it pulls the student toward the privileged teacher. \textbf{Novelty} ($s_t\!\equiv\!-1$) is the mirror direction: it pushes the student away. All probes use DAPO-Math-17k with Qwen3-1.7B and a GRPO backbone; on AIME24 rollouts, we track global reasoning ($\mathrm{pass@16}$, $\mathrm{acc@16}$), rigorous execution ($\mathrm{StepAcc}$ via ProcessBench~\citep{zheng2024processbench}), and flexible exploration ($E(y)$ epistemic-marker density, with response length as support).

\begin{figure}[H]
\centering
\widegraphics[0.86\linewidth]{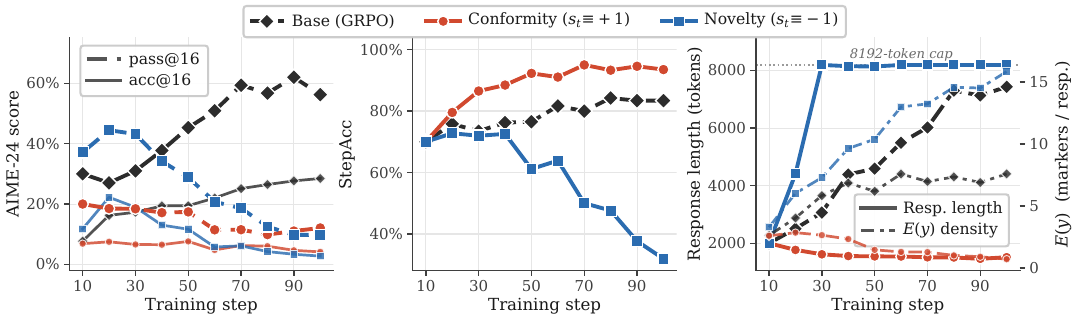}
\widecaption{Training trajectories of \emph{Base} (GRPO), \emph{Conformity} ($s_t\!\equiv\!{+}1$), and \emph{Novelty} ($s_t\!\equiv\!{-}1$) on AIME-24. (a)~Global reasoning ($\mathrm{pass@16}$, $\mathrm{acc@16}$). (b)~Rigorous execution ($\mathrm{StepAcc}$). (c)~Exploration ($E(y)$ density, response length). Conformity suppresses exploration without gaining accuracy; Novelty briefly overshoots Base then collapses into verbosity without reasoning.}
\label{fig:rq1-trajectories}
\end{figure}

\textbf{Finding 1: both uniform signs fail, but they fail on different axes.}
Figure~\ref{fig:rq1-trajectories}(a) shows that neither sign matches \emph{Base}: Conformity quickly loses $\mathrm{pass@16}$ and never improves $\mathrm{acc@16}$, while Novelty briefly rises above Base before collapsing. The component metrics explain why. Conformity leaves execution roughly intact but cuts response length and sharply suppresses $E(y)$, indicating lost exploration. Novelty does the opposite: $E(y)$ and length keep growing even after accuracy crashes, while $\mathrm{StepAcc}$ deteriorates. Thus the mirror hypothesis fails. Reasoning health is not a scalar controlled by one global sign; it requires both stable execution and flexible exploration.

\begin{figure}[H]
  \centering
  \includegraphics[width=0.66\linewidth]{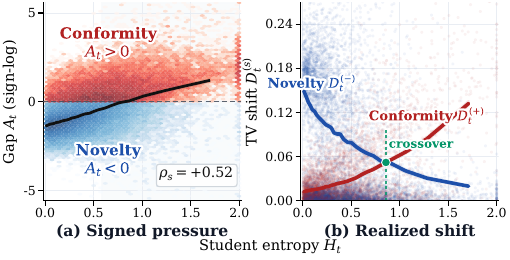}
  \caption{Token-level OPSD pressure vs.\ student entropy $H_t$. \textbf{(a)}~Log-evidence gap $A_t$ ($A_t{>}0$: Conformity;  $A_t{<}0$: Novelty), Spearman $\rho_s{=}{+}0.52$. \textbf{(b)}~TV shift $D_t^{(s)}$: Conformity displaces hardest at high-$H_t$ forks; Novelty at low-$H_t$ scaffolding.}
  \label{fig:rq2-unified}
\end{figure}

\refstepcounter{subsection}\label{sec:rq2}\noindent\textbf{Probe 2: Which token property separates the two failure modes?}
\noindent The complementary failures in Probe 1 (\S\ref{sec:rq1}) imply that Conformity and Novelty act most strongly on different token populations. Because privileged conditioning mainly removes residual uncertainty~\citep{kim2026whyopsd}, the student's predictive entropy $H_t=\mathcal{H}[\pi_\theta(\cdot\mid x,o_{<t})]$ is the natural candidate separator. We measure this relationship in two ways: $A_t$ captures the intended signed teacher pressure, while a one-step counterfactual TV shift captures the realized distributional displacement,
$D_t^{(s)} = \mathrm{TV}\!\big(\pi_\theta(\cdot\mid x,o_{<t}),\;\pi_\theta^{(s\text{-step})}(\cdot\mid x,o_{<t})\big),$
where $\pi_\theta^{(s\text{-step})}$ is the student after one gradient step with sign $s\in\{+1,-1\}$.

\textbf{Finding 2: the two signs concentrate pressure at opposite entropy regimes.}
Figure~\ref{fig:rq2-unified}(a) shows that strong Conformity pressure concentrates at high-$H_t$ tokens, while strong Novelty pressure concentrates at low-$H_t$ tokens; the signed gap has a positive Spearman correlation with entropy ($\rho={+}0.52$).
Figure~\ref{fig:rq2-unified}(b) shows the same split in realized movement: Conformity shifts the student most at high-$H_t$ forks, whereas Novelty shifts it most at low-$H_t$ scaffolding. This aligns with prior entropy analyses of reasoning: high-entropy forking tokens carry much of RLVR's search benefit, while low-entropy tokens execute deterministic scaffolding~\citep{wang2025highentropy,jin2026entropy}. Entropy therefore connects the two failures mechanistically: uniform attraction attacks the forks that should stay exploratory, while uniform repulsion attacks the scaffolding that should stay stable. Since $H_t$ is computed from the student alone, it is also a feasible routing variable.

\begin{table}[H]
  \centering
  \begingroup
  \small
  \setlength{\tabcolsep}{7pt}
  \renewcommand{\arraystretch}{1.05}
  \setlength{\aboverulesep}{0.35ex}
  \setlength{\belowrulesep}{0.45ex}
  \caption{\textbf{Prefix interventions} (\% change vs. unmodified rollouts). Green/red highlight the predicted target gains/failures; gray rows are entropy-agnostic controls.}
  \label{tab:rq3-direction-adaptive}
  \begin{tabular}{@{}>{\raggedright\arraybackslash}p{0.42\linewidth}rr@{}}
  \toprule
  \rowcolor{black!4}
  \textbf{Bucket / metric} & \textbf{Conformity} & \textbf{Novelty} \\
  \midrule
  \multicolumn{3}{@{}l}{\textbf{Low-}$\boldsymbol{H_t}$ \textbf{scaffolding}} \\
  $\Delta$StepAcc & \cellcolor{OliveGreen!22}\textcolor{OliveGreen!75!black}{\textbf{+18.3\%}} & \cellcolor{BrickRed!22}\textcolor{BrickRed!80!black}{\textbf{$-79.4\%$}} \\
  $\Delta E(y)$ & \textcolor{BrickRed}{$-4.21\%$} & \textcolor{OliveGreen!75!black}{$+5.03\%$} \\
  \addlinespace[1pt]
  \midrule
  \multicolumn{3}{@{}l}{\textbf{High-}$\boldsymbol{H_t}$ \textbf{forks}} \\
  $\Delta$StepAcc & \textcolor{BrickRed}{$-2.17\%$} & \textcolor{BrickRed}{$-6.18\%$} \\
  $\Delta E(y)$ & \cellcolor{BrickRed!22}\textcolor{BrickRed!80!black}{\textbf{$-86.2\%$}} & \cellcolor{OliveGreen!22}\textcolor{OliveGreen!75!black}{\textbf{+61.3\%}} \\
  \addlinespace[1pt]
  \midrule
  \rowcolor{black!7}\multicolumn{3}{@{}l}{\textbf{Random-}$\boldsymbol{H_t}$ \textbf{control}} \\
  \rowcolor{black!7}$\Delta$StepAcc & $+1.84\%$ & $-2.09\%$ \\
  \rowcolor{black!7}$\Delta E(y)$ & $-3.56\%$ & $+4.47\%$ \\
  \bottomrule
  \end{tabular}
  \endgroup
\end{table}

\refstepcounter{subsection}\label{sec:rq3}\noindent\textbf{Probe 3: Does entropy routing causally repair the mismatch?}
\noindent The entropy split suggests a diagonal repair: use Conformity on low-$H_t$ scaffolding and Novelty on high-$H_t$ forks. To test this directly, we perform a prefix-intervention probe at inference time, with no gradient update. At each position $t$, we bucket the token by student entropy ($H_t<\tau_{0.5}$ or $H_t\geq\tau_{0.5}$), where $\tau_{0.5}$ is the trajectory-local median entropy. Within the chosen bucket, we sample $o_t$ from one of two interventional distributions: Conformity substitutes the privileged teacher, $\tilde{\pi}_t^{\mathrm{C}}=\pi_\theta(\cdot\mid x,r,o_{<t})$, while Novelty softly suppresses teacher-aligned mass,
$\tilde{\pi}_t^{\mathrm{N}}(v) \propto \pi_\theta(v\mid x,o_{<t})\,/\,\pi_\theta(v\mid x,r,o_{<t})^{\alpha}, \alpha=0.5.$
After the intervened token, generation resumes under the unprivileged student. A Random-$H_t$ control applies the same interventions to entropy-agnostic token samples, isolating entropy from token frequency or position. We report changes in $\mathrm{StepAcc}$ and $E(y)$ relative to unmodified AIME24 rollouts.

\textbf{Finding 3: correct routing helps, while misrouting reproduces the OPSD failure.}
Table~\ref{tab:rq3-direction-adaptive} shows the desired diagonal pattern. On low-$H_t$ scaffolding, Conformity improves step accuracy by $18.3\%$ with only a small expected side-cost in exploration. On high-$H_t$ forks, Novelty increases epistemic-marker density by $61.3\%$ with only a modest step-accuracy cost. The off-diagonal cells show the complementary harms predicted by Probe 2 (\S\ref{sec:rq2}): Conformity on high-$H_t$ tokens collapses $E(y)$ by $86.2\%$, while Novelty on low-$H_t$ tokens drops $\mathrm{StepAcc}$ by $79.4\%$. The Random-$H_t$ control keeps all effects near zero ($|\Delta|\leq5\%$), ruling out token frequency and position as main explanations.

Together, the probes establish the design rule used next. OPSD fails because it applies one teacher-pressure direction to both entropy regimes. The repair is to route the same privileged self-teacher signal diagonally: attractive on low-$H_t$ scaffolding, repulsive on high-$H_t$ forks. Crucially, the router uses only the student's own entropy, so the method can exploit privileged guidance during training without requiring the privileged context at inference time. \emph{Detailed diagnostic text is provided in Appendix~\ref{app:sec3_detailed}.}\par\smallskip

\section{Direction-Adaptive Self-Distillation}\label{sec:method}

This section turns the token-level anatomy above into \textbf{Direction-Adaptive Self-Distillation} (\textbf{DASD}). The guiding principle is to keep the privileged self-teacher as a dense local source of evidence, while routing its pressure by the student's uncertainty instead of applying it as uniform imitation.

\noindent\textbf{Self-distillation objective with adaptive direction.}
The preceding probes leave a sharper requirement than reweighting: the same self-teacher signal must attract low-entropy scaffolding and repel high-entropy forks. We begin with a signed KL primitive that makes this direction explicit. At the prefix $(x,o_{<t}^{(i)})$, the inference-time student is $\pit{i}$, and the training-time privileged self-teacher is $\mathrm{sg}[\pitr{i}]$, where $\mathrm{sg}[\cdot]$ freezes the privileged branch during the update. Ordinary OPSD minimizes this KL and therefore makes the self-teacher uniformly attractive; DASD keeps the same local comparison but multiplies it by a signed coefficient:
\begin{equation}
\begin{aligned}
\widetilde{\mathcal{D}}_t^{(i)}(\theta)
&=\omega_t^{(i)}
\mathrm{KL}\!\left(
\pit{i}
\,\middle\|\,
\mathrm{sg}\!\left[\pitr{i}\right]
\right),
\qquad \omega_t^{(i)}\in[-1,1].
\end{aligned}
\label{eq:dasd-coupling}
\end{equation}
A positive $\omega_t^{(i)}$ preserves standard self-distillation: minimizing the loss reduces the KL and pulls the student toward the privileged self-teacher. A negative $\omega_t^{(i)}$ reverses this local pressure, making the self-teacher repulsive where imitation would prematurely collapse alternative continuations.

\noindent\textbf{Entropy-routed direction and gap-gated magnitude.}
DASD computes $\omega_t$ from rollout-measured quantities. The entropy term chooses the teacher-pressure direction, and the emitted-token log-evidence gap decides how strongly that directional signal should enter the update. Define
$\delta_t^{(i)}=\prv{\log\pi_\theta(o_t^{(i)}\mid x,r,o_{<t}^{(i)})}-\log\pi_\theta(o_t^{(i)}\mid x,o_{<t}^{(i)})$
as the privileged self-teacher's log-probability advantage on the sampled token. Let $\tau_\rho^{(i)}$ be the $\rho$-th percentile of the trajectory entropies, let $\hat\sigma_H^{(i)}=\mathrm{MeanAbsDev}_t(\{H_t^{(i)}\})$ be their mean absolute deviation, and let $\tilde\delta^{(i)}=\operatorname{median}_t|\delta_t^{(i)}|$. With $\bar\delta_t^{(i)}=\delta_t^{(i)}/(\tilde\delta^{(i)}+\epsilon)$ and logistic sigmoid $\sigma(\cdot)$, DASD sets:
\begin{equation}
\omega_t^{(i)}=
\underbrace{\tanh\!\left(\tfrac{\tau_\rho^{(i)}-H_t^{(i)}}{\hat\sigma_H^{(i)}+\epsilon}\right)}_{\text{\scriptsize entropy router}}
\cdot
\underbrace{\sigma\!\left(|\bar\delta_t^{(i)}|-1\right)}_{\text{\scriptsize gap gate}},
\quad
\tau_\rho^{(i)}=\operatorname{Quantile}_{\rho}\!\left(\{H_t^{(i)}\}_t\right).
\label{eq:zeta}
\end{equation}
The first factor is positive below the trajectory percentile and negative above it, matching the diagonal rule in Table~\ref{tab:rq3-direction-adaptive}. Rather than imposing a hard boundary, the $\tanh$ router smoothly changes the self-teacher from attractive on low-$H_t$ scaffolding to repulsive on high-$H_t$ forks. The sigmoid gate then attenuates teacher--student fluctuations smaller than the trajectory's own gap scale, so only reliable disagreements shape the signed token pressure.

\noindent\textbf{Verifier-anchored directional objective.}
The signed teacher field specifies where local pressure should point, but it does not decide whether a full reasoning path is correct. DASD therefore keeps the verifier as the global anchor. With $A_G^{(i)}=(R(x,o^{(i)})-\mu_G)/\sigma_G$ denoting the per-prompt group-relative verifier advantage over the $G$ rollouts for prompt $x$, the distribution-level loss is
\begin{equation}
\resizebox{.945\linewidth}{!}{$
\ell_{\mathrm{DASD}}(\theta)
= -\mathbb{E}\!\left[
\frac{1}{G}\sum_{i=1}^{G}\frac{1}{|o^{(i)}|}
\sum_{t=1}^{|o^{(i)}|}
\min\!\left(
\rho_t^{(i)} A_G^{(i)},
\clip(\rho_t^{(i)},1-\epsilon,1+\epsilon)A_G^{(i)}
\right)
\right]
+
\beta\,\mathbb{E}\!\left[
\frac{1}{G}\sum_{i=1}^{G}\frac{1}{|o^{(i)}|}\sum_{t=1}^{|o^{(i)}|}
\widetilde{\mathcal{D}}_t^{(i)}(\theta)
\right].
$}
\label{eq:dasd-functional}
\end{equation}
The first term supplies trajectory-level correctness; the second distributes signed token-level pressure inside verifier-scored rollouts. The privileged branch therefore refines credit assignment without replacing the verifier or being used at inference time.

\smallskip
\noindent\textbf{Sampled realization for PPO training.}
Equation~\eqref{eq:dasd-functional} states the desired local geometry; PPO training requires a sampled estimator. DASD uses the emitted-token log-evidence gap rather than an unrestricted full-vocabulary imitation target, keeping privileged information in a controlled token-level correction.

\begin{proposition}[Sampled realization of the signed reference field]\label{prop:dasd-credit}
Fix $x,o_{<t},r$ and a sampled token $o_t$. Let $p_t=\pi_\theta(\cdot\mid x,o_{<t})$, let $q_t=\mathrm{sg}[\pi_\theta(\cdot\mid x,r,o_{<t})]$, and define $\delta_t=\log q_t(o_t)-\log p_t(o_t)$. Treating $q_t$, $A_G$, and the measured routing coefficient $\omega_t$ as fixed during the current update,
\[
\ell_t(\theta)=-A_G\log p_t(o_t)+\beta\,\omega_t\,\mathrm{KL}(p_t\|q_t),
\quad
\nabla_\theta\ell_t
\stackrel{\mathrm{MC}}{\doteq}
-\big(A_G+\beta\,\omega_t\delta_t\big)\nabla_\theta\log p_t(o_t).
\]
Here $\doteq$ denotes Monte Carlo score-function form after dropping action-independent baseline terms.
\end{proposition}

\noindent Appendix~\ref{app:derivation} gives the proof. Proposition~\ref{prop:dasd-credit} shows that the distributional signed field reduces, on sampled tokens, to an additive credit term: the verifier contributes the trajectory anchor $A_G$, while the privileged branch contributes only the directional correction $\beta\omega_t\delta_t$. For numerical stability, DASD normalizes this correction within each trajectory and uses
$\hat A_t^{(i)}=A_G^{(i)}+\beta\omega_t^{(i)}\bar\delta_t^{(i)}$.
The final update maximizes the standard clipped PPO surrogate with this direction-adaptive advantage. Writing $\ell_{\mathrm{clip}}(\rho,a)=\min(\rho a,\clip(\rho,1-\epsilon_{\mathrm{clip}},1+\epsilon_{\mathrm{clip}})a)$, DASD averages $\ell_{\mathrm{clip}}(\rho_t^{(i)},\hat A_t^{(i)})$ over sampled tokens and trajectories; Algorithm~\ref{alg:dasd} (Appendix~\ref{app:algorithm}) summarizes the procedure. The limiting cases make the geometry transparent: $\beta=0$ recovers verifier-only learning, $\omega_t\equiv+1$ recovers uniform attraction, $\omega_t\equiv-1$ recovers uniform repulsion, and DASD selects the diagonal implied by the probes---attractive where the path is settled, repulsive where premature certainty would erase search.

\section{Experiments}\label{sec:experiments}

\begin{table}[H]
  \centering
  \begingroup
  \small
  \setlength{\tabcolsep}{8.5pt}
  \renewcommand{\arraystretch}{1.03}
  \setlength{\aboverulesep}{0.25ex}
  \setlength{\belowrulesep}{0.35ex}
  \caption{Main results (Avg@16 accuracy, \%; higher is better) across six reasoning benchmarks and three model scales;
    \textbf{bold}/\underline{underline} mark best/second-best within each scale.
    Full results in Table~\ref{tab:full_results}.}
  \label{tab:main_results}
  \widetablebox{%
  \begin{tabular}{@{}lccccccc@{}}
  \toprule
  \rowcolor{black!3}
  \textbf{Method}
    & \multicolumn{1}{c}{\textbf{MATH500}}
    & \multicolumn{1}{c}{\textbf{Minerva}}
    & \multicolumn{1}{c}{\textbf{HMMT25}}
    & \multicolumn{1}{c}{\textbf{Olympiad}}
    & \multicolumn{1}{c}{\textbf{AIME24}}
    & \multicolumn{1}{c}{\textbf{AIME25}}
    & \multicolumn{1}{c}{\textbf{Avg}} \\
  \midrule
  \textit{Qwen3-1.7B}                                        & 58.4 & 17.5 & 7.9 & 33.4 & 13.1 & 10.6 & 23.5 \\
  \quad +GRPO~\citep{shao2024deepseekmath}                   & 68.2 & 22.6 & 11.5 & 40.8 & 32.5 & 27.5 & 33.8 \\
  \quad +HEPO~\citep{wang2025highentropy}                    & 68.4 & 22.8 & \underline{15.6} & \underline{42.8} & 32.5 & \underline{29.2} & \underline{35.2} \\
  \quad +OPSD~\citep{zhao2026opsd}                           & 59.2 & 17.5 & 8.5 & 32.5 & 15.0 & 9.2 & 23.6 \\
  \quad +SDPO~\citep{hubotter2026sdpo}                       & 58.8 & 17.8 & 8.1 & 32.8 & 14.2 & 9.8 & 23.6 \\
  \quad +RLSD~\citep{yang2026rlsd}                           & \underline{68.7} & \underline{25.1} & 15.2 & 39.6 & \textbf{35.4} & 26.7 & 35.1 \\
  \quad +SRPO~\citep{li2026srpo}                             & 62.8 & 20.1 & 10.2 & 36.6 & 23.3 & 17.1 & 28.4 \\
  \rowcolor{bsdcol!8}
  \quad \textbf{+DASD (ours)}                                & \textbf{70.4} & \textbf{25.3} & \textbf{17.7} & \textbf{45.4} & \underline{35.0} & \textbf{32.1} & \textbf{37.7} \\
  \midrule
  \textit{Qwen3-4B}                                          & 67.0 & 24.6 & 15.2 & 43.1 & 23.1 & 18.5 & 31.9 \\
  \quad +GRPO~\citep{shao2024deepseekmath}                   & 72.6 & 28.2 & \underline{22.1} & 47.0 & 55.0 & \underline{49.8} & \underline{45.8} \\
  \quad +HEPO~\citep{wang2025highentropy}                    & 72.7 & 27.1 & 21.2 & \underline{50.4} & 53.3 & 45.2 & 45.0 \\
  \quad +OPSD~\citep{zhao2026opsd}                           & 66.4 & 24.8 & 12.9 & 43.2 & 24.2 & 20.0 & 31.9 \\
  \quad +SDPO~\citep{hubotter2026sdpo}                       & 66.9 & 24.4 & 13.3 & 43.0 & 25.0 & 19.4 & 32.0 \\
  \quad +RLSD~\citep{yang2026rlsd}                           & \underline{73.0} & \underline{28.5} & 20.4 & 44.9 & \underline{56.5} & 47.7 & 45.2 \\
  \quad +SRPO~\citep{li2026srpo}                             & 69.6 & 26.1 & 17.7 & 45.0 & 40.0 & 34.4 & 38.8 \\
  \rowcolor{bsdcol!8}
  \quad \textbf{+DASD (ours)}                                & \textbf{75.0} & \textbf{31.0} & \textbf{25.0} & \textbf{57.5} & \textbf{58.3} & \textbf{50.0} & \textbf{49.5} \\
  \midrule
  \textit{Qwen3-8B}                                          & 67.2 & 24.7 & 13.1 & 44.3 & 27.9 & 21.0 & 33.0 \\
  \quad +GRPO~\citep{shao2024deepseekmath}                   & \underline{73.0} & 27.7 & \underline{20.8} & 49.4 & 59.4 & 48.8 & 46.5 \\
  \quad +HEPO~\citep{wang2025highentropy}                    & \underline{73.0} & 28.7 & 20.4 & \underline{50.3} & \underline{66.7} & 49.0 & \underline{48.0} \\
  \quad +OPSD~\citep{zhao2026opsd}                           & 67.5 & 23.8 & 15.0 & 43.6 & 27.5 & 19.6 & 32.8 \\
  \quad +SDPO~\citep{hubotter2026sdpo}                       & 67.0 & 24.1 & 14.4 & 43.9 & 27.9 & 20.4 & 33.0 \\
  \quad +RLSD~\citep{yang2026rlsd}                           & \underline{73.0} & \underline{29.7} & \underline{20.8} & 49.4 & 62.3 & \underline{49.8} & 47.5 \\
  \quad +SRPO~\citep{li2026srpo}                             & 70.0 & 25.9 & 17.5 & 46.8 & 43.8 & 34.6 & 39.7 \\
  \rowcolor{bsdcol!8}
  \quad \textbf{+DASD (ours)}                                & \textbf{76.0} & \textbf{33.0} & \textbf{25.0} & \textbf{55.8} & \textbf{71.7} & \textbf{53.3} & \textbf{52.5} \\
\bottomrule
  \end{tabular}}
  \endgroup
\end{table}

The experiments evaluate DASD as a reasoning post-training method. We first measure direct mathematical accuracy under a shared RLVR/self-distillation protocol, then test whether improvements persist under larger sampling budgets. We next check whether the resulting policy preserves both ingredients of healthy reasoning: rigorous step execution and flexible self-correction. Finally, we evaluate DASD on non-math domains and additional model families.

\noindent\textbf{Models and Training Datasets.}
We use Qwen3-1.7B~\citep{yang2025qwen3} for detailed diagnostics and additionally report Qwen3-4B and Qwen3-8B in the main benchmark table to test scale consistency.
For self-distillation, the student and privileged self-teacher share weights; the teacher is the same model conditioned on a verified training trace $r$, while the student samples without $r$.
All methods are trained on DAPO-Math-17k~\citep{yu2026dapo}, using its verifiable answers for privileged-context and rewards.

\noindent\textbf{Evaluation Benchmarks and Metrics.}
We evaluate on MATH500~\citep{hendrycks2021math}, Minerva~\citep{lewkowycz2022solving}, HMMT25, OlympiadBench~\citep{he2024olympiadbench}, AIME24, and AIME25, spanning standard math benchmarks through hard competition-style reasoning problems.
Unless otherwise stated, we sample $K=16$ rollouts per question with temperature 1.0, top-$p=0.9$, and an 8{,}192-token limit, and report Avg@16 and Pass@16.
For AIME, we also report Pass@$k$ for $k\in\{1,4,8,16,32,64,128\}$ to test whether gains reflect broader coverage under larger sampling budgets.
As a compact diagnostic of reasoning quality, we measure rigorous execution (StepAcc, FES, CSR) and flexible exploration ($E(y)$ density, RevRate, Dist-3) on AIME24 rollouts at convergence; definitions are given in Appendix~\ref{app:exp_details}.

\noindent\textbf{Baselines and Implementation.}
We compare against GRPO~\citep{shao2024deepseekmath}, HEPO~\citep{wang2025highentropy}, OPSD~\citep{zhao2026opsd}, SDPO~\citep{hubotter2026sdpo}, RLSD~\citep{yang2026rlsd}, and SRPO~\citep{li2026srpo}; Appendix~\ref{app:exp_details} summarizes what each baseline changes relative to RLVR or self-distillation.
All methods use the \texttt{verl}~\citep{sheng2024hybridflow} framework with batch size $B=128$, mini-batch size $B_{\mathrm{mini}}=32$, learning rate $3\times10^{-6}$, and cosine decay. All methods use $G=8$ rollouts per prompt with per-prompt group-relative verifier advantages.
DASD uses router quantile $\rho=0.20$ by default, so $\tau_{\rho}^{(i)}$ is recomputed as the trajectory entropy quantile for each rollout; The privileged context $r$ is formed by prepending the verified ground-truth solution to the prompt.

\textbf{Performance on Mathematical Reasoning.}
Table~\ref{tab:main_results} shows that DASD achieves the best macro Avg@16 at all three model scales.
The strongest margins appear mainly on harder competition-style benchmarks---HMMT25, OlympiadBench, AIME24, and AIME25---where successful reasoning requires search over multiple plausible solution paths.
By contrast, OPSD and SDPO are often close to the base model on average, despite using privileged self-teacher signals---indicating that dense self-distillation alone is not sufficient.
Additional results are in Appendix~\ref{app:full_results}.

\begin{figure}[H]
  \centering
  \widegraphics[0.87\linewidth]{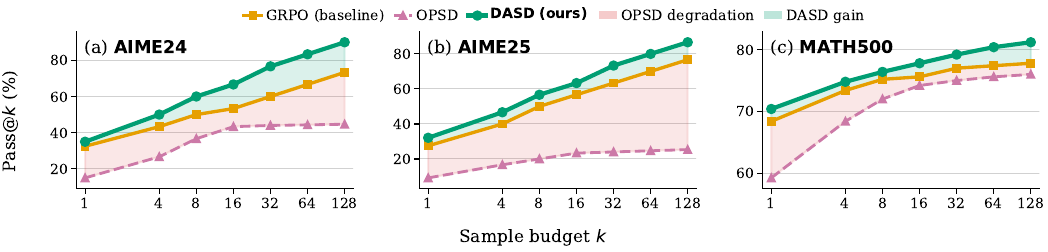}
  \widecaption{Pass@$k$ performance on AIME24, AIME25, and MATH500. DASD stays above GRPO and OPSD across all sample budgets, with the separation growing at larger $k$---indicating that DASD produces a more diverse set of reasoning trajectories.}
  \label{fig:pass_main}
\end{figure}

\begin{table}[H]
  \centering
  \begingroup
  \small
  \setlength{\tabcolsep}{8pt}
  \renewcommand{\arraystretch}{1.05}
  \setlength{\aboverulesep}{0.35ex}
  \setlength{\belowrulesep}{0.45ex}
  \caption{\textbf{Reasoning health} at AIME24 convergence (Qwen3-1.7B). Higher is better; DASD is shaded, and red marks OPSD's exploration collapse.}
  \label{tab:epistemic_recovery}
  \begin{tabular}{@{}>{\raggedright\arraybackslash}p{0.39\linewidth}rr>{\columncolor{bsdcol!8}}r@{}}
  \toprule
  \rowcolor{black!4}
  \textbf{Metric} & \textbf{GRPO} & \textbf{OPSD} & \textbf{DASD} \\
  \midrule
  \multicolumn{4}{@{}l}{\textbf{Execution quality}} \\
  Step accuracy $\uparrow$ & 58.4 & \underline{59.1} & \textbf{63.7} \\
  First-error step $\uparrow$ & 4.2 & \underline{4.4} & \textbf{5.1} \\
  Correct-step ratio $\uparrow$ & 61.3 & \underline{62.0} & \textbf{66.8} \\
  \addlinespace[1pt]
  \midrule
  \multicolumn{4}{@{}l}{\textbf{Exploration quality}} \\
  $E(y)$ density $\uparrow$ & 3.8 & \textcolor{BrickRed}{0.7} & \textbf{4.3} \\
  Revision rate $\uparrow$ & 0.41 & \textcolor{BrickRed}{0.06} & \textbf{0.47} \\
  Distinct-3 $\uparrow$ & 0.72 & \textcolor{BrickRed}{0.51} & \textbf{0.81} \\
  \addlinespace[1pt]
  \midrule
  \multicolumn{4}{@{}l}{\textbf{Outcome anchor}} \\
  Avg@16 $\uparrow$ & 32.5 & 15.0 & \textbf{35.0} \\
  \bottomrule
  \end{tabular}
  \endgroup
\end{table}

\textbf{Pass@$k$ Scaling.}
Fixed-budget accuracy does not show whether a method merely improves its first attempt or also expands the set of reachable solution paths. Figure~\ref{fig:pass_main} therefore evaluates Pass@$k$ on AIME24, AIME25, and MATH500. DASD outperforms GRPO and OPSD across the sampling range, with the gap widening as $k$ increases. On AIME24, for example, DASD starts close to GRPO at Pass@1, reaches $66.7\%$ at $k=16$, and continues to $83.3\%$ at $k=64$ and $90.0\%$ at $k=128$, while OPSD plateaus near $44\%$ after $k=16$. The widening gap indicates that DASD improves coverage of diverse reasoning trajectories rather than only sharpening a single-sample prediction.

\textbf{Reasoning Health Diagnosis.}
Finally, we check whether the performance gains preserve the two reasoning-health axes introduced above: rigorous execution and flexible exploration. Table~\ref{tab:epistemic_recovery} reports these diagnostics at convergence on AIME24 rollouts (Qwen3-1.7B).

Table~\ref{tab:epistemic_recovery} shows the trade-off hidden by aggregate accuracy. OPSD slightly improves local execution over GRPO, but collapses exploration: $E(y)$ density, RevRate, and Dist-3 all fall sharply. DASD is the only method that improves both sides, exceeding GRPO and OPSD on execution while also surpassing GRPO on exploration. Thus the task-performance gains in Table~\ref{tab:main_results} do not come from sacrificing self-correction or trajectory diversity.

\begin{table}[H]
  \centering
  \begingroup
  \small
  \setlength{\tabcolsep}{6pt}
  \renewcommand{\arraystretch}{1.02}
  \setlength{\aboverulesep}{0.25ex}
  \setlength{\belowrulesep}{0.35ex}
  \caption{\textbf{Cross-domain and cross-family evaluation.}
  Domain columns are Qwen3-8B averages; family columns are math averages with model scales around 8B.
  Full results are in Appendix~\ref{app:cross_domain_family}.}
  \label{tab:cross_domain_family_main}
  \begin{tabular}{@{}lrrrrrr@{}}
  \toprule
  \rowcolor{black!4}
  \textbf{Method}
    & \multicolumn{3}{c}{\textbf{Domain}}
    & \multicolumn{3}{c}{\textbf{Family}} \\
  \cmidrule(lr){2-4}\cmidrule(l){5-7}
    & \textbf{Code} & \textbf{Sci.} & \textbf{Tool}
    & \textbf{Qwen3} & \textbf{Olmo} & \textbf{Llama} \\
  \midrule
  Base        & 72.1 & 62.5 & 69.6 & 38.7 & 33.8 & 31.1 \\
  +GRPO       & 72.3 & 63.0 & 70.1 & 60.4 & 49.4 & 44.7 \\
  +OPSD       & 69.9 & 64.3 & 70.3 & 38.2 & 34.5 & 31.4 \\
  \rowcolor{bsdcol!8}
  \textbf{+DASD}
              & \textbf{73.7} & \textbf{65.2} & \textbf{71.1}
              & \textbf{67.0} & \textbf{54.2} & \textbf{49.1} \\
  \bottomrule
  \end{tabular}
  \endgroup
\end{table}

\textbf{Cross-Domain and Cross-Family Evaluation.}
Table~\ref{tab:cross_domain_family_main} asks whether the gains extend beyond the main Qwen3 math setting. For cross-domain evaluation, we keep the Qwen3-8B family fixed and average two benchmarks per domain: code, scientific reasoning, and tool use. For cross-family evaluation, we keep the math domain fixed and compare Qwen3, Olmo, and Llama backbones. DASD improves every domain average and every family average, while OPSD is often close to the base model in most settings. These results suggest that DASD is not only a Qwen3 math-specific improvement; full result details are in Appendix~\ref{app:cross_domain_family}.

\section{Analysis}\label{sec:analysis}

The experiments above show that DASD improves reasoning while preserving both execution and exploration. We now test the mechanism with three targeted checks: whether the entropy substrate survives through training and at convergence, whether the protected forks and revisions are causally useful, and whether the specific routing outperforms beats plausible replacements.

\begin{figure}[H]
  \centering
  \widegraphics[0.84\linewidth]{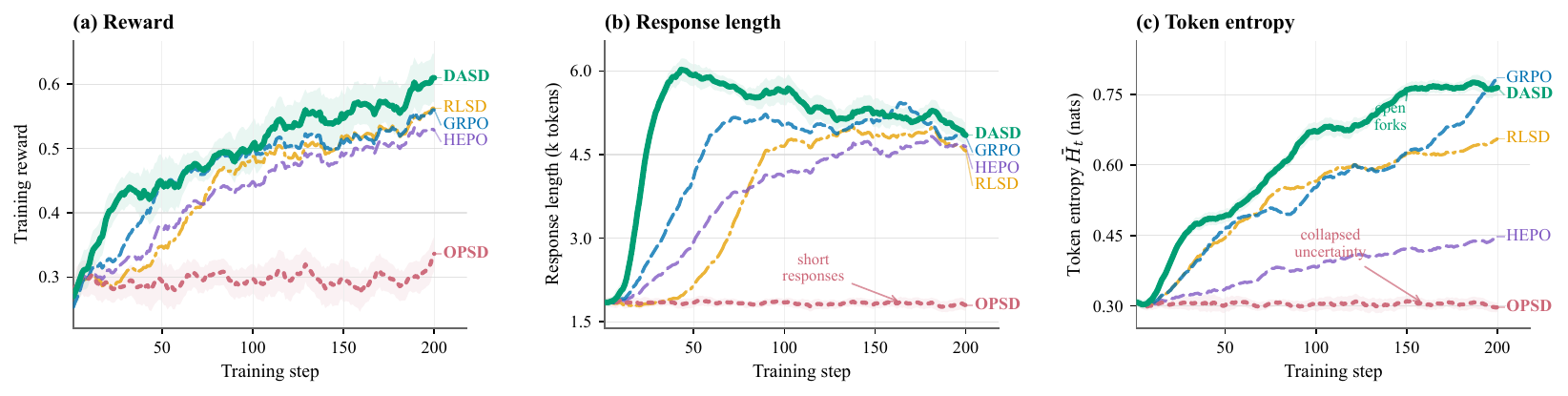}
  \widecaption{Training dynamics on Qwen3-1.7B over 200 updates.
  \textbf{(a)}~Training reward.
  \textbf{(b)}~Mean response length.
  \textbf{(c)}~Mean token-level entropy $\bar{H}_t$;
  DASD couples fast reward growth with sustained response length and preserved token entropy, whereas OPSD compresses both length and uncertainty.}
  \label{fig:training_dynamics}
\end{figure}

\refstepcounter{subsection}\label{sec:routing_verification}
\noindent\textbf{RQ1: Does DASD preserve the entropy substrate, and do the two routing arms separate roles?}
DASD needs a meaningful high-entropy tail: if training collapses all tokens into low uncertainty, there are no forks left for the repulsive arm to protect. We therefore analyze the substrate at two time scales: training trajectories for reward, length, and entropy, followed by endpoint entropy density and selective flips of either routing arm.

Figure~\ref{fig:training_dynamics} shows the optimization path behind the mechanism. DASD raises reward while maintaining long responses and non-collapsed token entropy; OPSD instead remains low-reward and compresses both length and entropy. Thus DASD does not merely reach a better endpoint: it preserves the uncertainty substrate during training. We next test whether the remaining uncertainty has the predicted high-$H_t$ fork structure and whether each routing arm supports the intended competence.

\begin{figure}[H]
  \centering
  \widegraphics[0.86\linewidth]{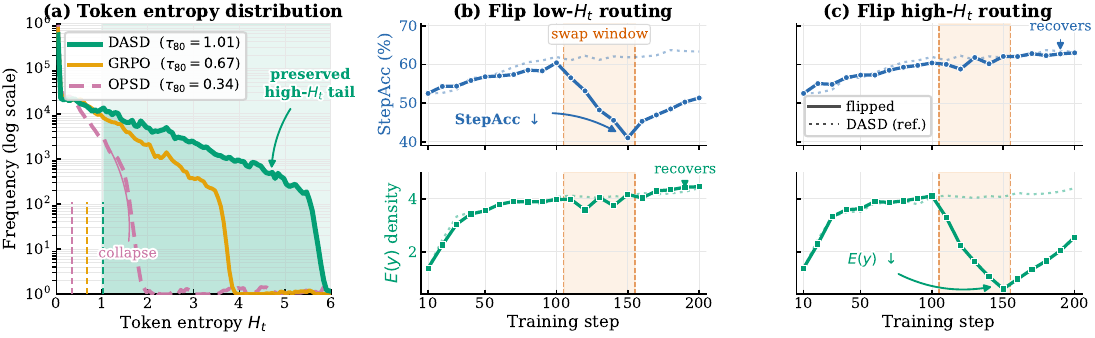}
  \widecaption{Entropy preservation and selective routing disruption.
    \textbf{(a)}~Log-frequency entropy curves compare the endpoint token-entropy distributions of OPSD, GRPO, and DASD.
    \textbf{(b)}~Flipping the low-$H_t$ routing arm primarily affects StepAcc, indicating its role in low-uncertainty execution tokens.
    \textbf{(c)}~Flipping the high-$H_t$ routing arm primarily affects $E(y)$, indicating its role in high-uncertainty exploratory forks. Dotted curves denote unperturbed DASD.}
  \label{fig:routing}
\end{figure}

Figure~\ref{fig:routing}(a) verifies the prerequisite for DASD's mechanism: unlike OPSD, DASD keeps a visible high-$H_t$ tail where alternative continuations remain available, with the 80th-percentile ordering OPSD $<$ GRPO $<$ DASD. The swap results then identify the role of each arm. Flipping low-$H_t$ routing hurts StepAcc while $E(y)$ only fluctuates before recovering, whereas flipping high-$H_t$ routing suppresses $E(y)$ while StepAcc only fluctuates before recovering. Thus DASD both preserves the fork population and assigns its two routing arms to distinct reasoning competences: execution for low-$H_t$ scaffolding and exploration for high-$H_t$ forks.

\begin{figure}[H]
  \centering
  \widegraphics[0.88\linewidth]{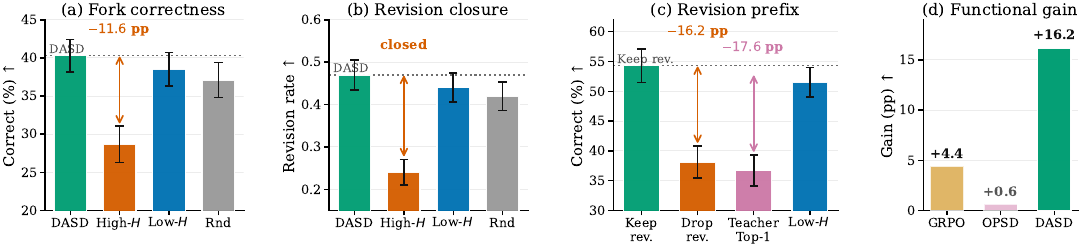}
  \widecaption{Causal intervention tests on DASD forks and revisions.
    \textbf{(a,b)}~Replacing selected tokens with the privileged teacher's top choice measures the effect of high-$H_t$, low-$H_t$, and random-position interventions on correctness and revision behavior.
    \textbf{(c)}~At matched DASD revision prefixes, preserving, suppressing, or teacher-forcing the revision continuation tests its effect on final correctness.
    \textbf{(d)}~The preserve--suppress comparison is shown across DASD, GRPO, and OPSD revisions.}
  \label{fig:causal_interventions}
\end{figure}

\refstepcounter{subsection}\label{sec:causal_interventions}
\noindent\textbf{RQ2: Are the protected forks and the revisions they enable causally useful?}
RQ1 shows that DASD preserves exploratory forks, but not whether those forks matter for final answers. We test this in two complementary ways. First, we intervene on \emph{where} DASD acts by replacing the highest-$H_t$ repulsive token with the privileged teacher's top choice, using low-$H_t$ and random-position replacements as controls. Second, we intervene on \emph{what behavior} matters at such prefixes by preserving or suppressing revision continuations.

Figure~\ref{fig:causal_interventions}(a,b) shows that high-$H_t$ forks are load-bearing: teacher-forcing them both reduces final correctness and closes the revision behavior that DASD preserves, whereas low-$H_t$ and random controls are much weaker. Figure~\ref{fig:causal_interventions}(c,d) shows why these forks matter: holding the prefix fixed, preserving DASD's revision continuation yields a large causal gain, while GRPO and OPSD revisions provide little benefit. Together, the interventions show that DASD's repulsive arm protects decision points where functional self-correction changes the final answer.

\refstepcounter{subsection}\label{sec:ablation_design}
\noindent\textbf{RQ3: Are DASD's design choices better than plausible replacements?}
DASD makes three design choices: student entropy as the routing signal, a smooth signed direction map, and a reliability gate on teacher--student gaps. Table~\ref{tab:design_space} tests compact replacements for each choice; full metrics are in Appendix~\ref{app:additional_analysis}.

\begin{table}[H]
  \centering
  \begingroup
  \small
  \setlength{\tabcolsep}{5pt}
  \renewcommand{\arraystretch}{1.02}
  \setlength{\aboverulesep}{0.35ex}
  \setlength{\belowrulesep}{0.45ex}
  \caption{\textbf{DASD design-space ablation} (Qwen3-1.7B). AIME columns report Avg@16; $E(y)$ is exploration density. Shaded rows are the selected DASD defaults; full grid in Appendix~\ref{app:additional_analysis}.}
  \label{tab:design_space}
  \begin{tabular}{@{}>{\raggedright\arraybackslash}p{0.42\linewidth}
    S[table-format=2.1]
    S[table-format=2.1]
    S[table-format=1.1]@{}}
  \toprule
  \rowcolor{black!4}
  \textbf{Variant}
    & {\textbf{AIME24}$\uparrow$}
    & {\textbf{AIME25}$\uparrow$}
    & {\textbf{$E(y)$}$\uparrow$} \\
  \midrule
  \multicolumn{4}{@{}l}{\textbf{A. Routing signal}} \\
  Position proxy & 17.8 & 15.9 & 1.6 \\
  Token frequency & 20.4 & 17.7 & 2.1 \\
  Gradient norm & 25.1 & 22.8 & 2.8 \\
  Attention entropy & \underline{30.9} & \underline{27.9} & \underline{3.9} \\
  \rowcolor{bsdcol!8}
  \textbf{Student entropy $H_t$} & \textbf{35.0} & \textbf{32.1} & \textbf{4.3} \\
  \addlinespace[1pt]
  \midrule
  \multicolumn{4}{@{}l}{\textbf{B. Direction map}} \\
  OPSD: const.\ $+1$ & 15.0 & 9.2 & 0.7 \\
  Novelty: const.\ $-1$ & 12.0 & 10.6 & \textbf{5.6} \\
  Hard threshold & 31.5 & 28.4 & 4.0 \\
  Linear ramp & \underline{33.6} & \underline{30.6} & 4.1 \\
  \rowcolor{bsdcol!8}
  \textbf{$\tanh$ direction} & \textbf{35.0} & \textbf{32.1} & \underline{4.3} \\
  \addlinespace[1pt]
  \midrule
  \multicolumn{4}{@{}l}{\textbf{C. Reliability gate}} \\
  No gate & \underline{29.4} & \underline{26.3} & \underline{3.8} \\
  Fixed threshold & 28.6 & 25.6 & 3.7 \\
  Magnitude-only & 30.8 & 27.7 & 4.0 \\
  \rowcolor{bsdcol!8}
  \textbf{Gap-reliability gate} & \textbf{35.0} & \textbf{32.1} & \textbf{4.3} \\
  \bottomrule
  \end{tabular}
  \endgroup
\end{table}

The ablation supports the full prescription rather than an arbitrary reweighting heuristic. Student entropy outperforms positional, gradient, and attention proxies; smooth signed routing improves over uniform signs and hard boundaries; and the gap gate prevents weak teacher disagreements from dominating the update. Combined with the interventions above, this shows that DASD works by preserving high-entropy forks, using them for functional revision, and applying teacher pressure only where its direction is appropriate.


\FloatBarrier
\section{Related Work}\label{sec:related}

\noindent\textbf{Entropy-guided policy optimization.}
A growing body of verifier-only refinements uses token entropy to reshape policy-gradient credit, e.g.\ restricting updates to high-entropy ``forking'' tokens~\citep{wang2025highentropy}, adding detached entropy bonuses for exploration~\citep{DBLP:conf/aaai/ChengHZDZZW26}, combining reward polarity with entropy in EAPO~\citep{he2026rethinkingtokenlevelcreditassignment}, probing signed RLVR directions in log-probability space~\citep{huang2026directionrlvrupdatesllm}, suppressing destabilizing low-probability low-entropy tokens~\citep{liu2026stapostabilizingreinforcementlearning}, or reweighting low-entropy segments by correctness~\citep{chen2025highentropyexplorationcorrectnessawarelowentropy}. All of these treat entropy as a \emph{magnitude} knob---it decides which tokens receive more or less RLVR pressure, while the supervision direction stays fixed by the verifier. DASD instead uses student entropy as a \emph{direction} router for privileged self-teacher supervision, attracting low-entropy scaffolding toward the teacher and repelling high-entropy forks away from it. To our knowledge, DASD is the first method to turn entropy from a weight on token updates into the sign selector for privileged teacher pressure.

\noindent\textbf{On-policy self-distillation (OPSD).}
On-policy self-distillation removes the external-teacher dependency of on-policy distillation~\citep{agarwal2024onpolicy} by reusing the same model as both student and teacher conditioned on privileged information~$r$~\citep{zhao2026opsd,hubotter2026sdpo,yang2026rlsd,li2026srpo,penaloza2026pidistill,stein2026gates,zhang2026piplay}. Recent analyses show, however, that uniform attraction to a solution-conditioned self-teacher suppresses epistemic markers, leaks privileged style, and collapses out-of-distribution reasoning~\citep{kim2026whyopsd,yang2026rlsd,fu2026revisiting}. Existing OPSD variants debate the magnitude, filtering, gating, scheduling, or location of teacher pressure, but never \emph{which direction} the pressure should point. DASD makes this pressure \emph{bipolar}: attractive on low-entropy scaffolding to stabilize execution and repulsive on high-entropy forks to preserve search, with the sign chosen from student entropy alone and no privileged input required at inference. Appendix~\ref{app:related_work} contains the full comparison to all relevant OPSD researches, entropy-based credit, and reasoning-elicitation methods.

\section{Conclusion}\label{sec:conclusion}

We diagnose OPSD's reasoning degradation as a problem of supervision \emph{direction}, not of privileged self-distillation itself: uniform attraction to the solution-conditioned teacher suppresses high-entropy forks, while uniform repulsion corrupts low-entropy scaffolding. Direction-Adaptive Self-Distillation (DASD) resolves this mismatch by letting student entropy $H_t$ route teacher pressure---pulling low-$H_t$ scaffolding toward the teacher to stabilize execution, pushing high-$H_t$ forks away to preserve exploration, and leaving the verifier as the global correctness anchor. Across six mathematical reasoning benchmarks and three model scales, DASD improves accuracy, broadens Pass@$k$ coverage, and recovers both rigorous step execution and flexible self-correction. These results indicate that the productive lever for self-distilled reasoning is the \emph{direction} of supervision, not merely its strength. Limitations and broader implications are discussed in Appendix~\ref{app:limitations_discussion}.

\bibliographystyle{unsrtnat}
{
\bibliography{neurips_2026}

\begin{thebibliography}{72}
\providecommand{\natexlab}[1]{#1}
\providecommand{\url}[1]{\texttt{#1}}
\expandafter\ifx\csname urlstyle\endcsname\relax
  \providecommand{\doi}[1]{doi: #1}\else
  \providecommand{\doi}{doi: \begingroup \urlstyle{rm}\Url}\fi

\bibitem[Hinton et~al.(2015)Hinton, Vinyals, and Dean]{hinton2015distilling}
Geoffrey Hinton, Oriol Vinyals, and Jeff Dean.
\newblock Distilling the knowledge in a neural network.
\newblock \emph{NeurIPS Deep Learning Workshop, arXiv preprint arXiv:1503.02531}, 2015.

\bibitem[Kim and Rush(2016)]{kim2016sequence}
Yoon Kim and Alexander~M. Rush.
\newblock Sequence-level knowledge distillation.
\newblock In \emph{Proceedings of the 2016 Conference on Empirical Methods in Natural Language Processing (EMNLP)}, 2016.

\bibitem[Furlanello et~al.(2018)Furlanello, Lipton, Tschannen, Itti, and Anandkumar]{furlanello2018born}
Tommaso Furlanello, Zachary~C. Lipton, Michael Tschannen, Laurent Itti, and Anima Anandkumar.
\newblock Born-again neural networks.
\newblock In \emph{Proceedings of the 35th International Conference on Machine Learning (ICML)}, 2018.

\bibitem[Xu et~al.(2024)Xu, Li, Tao, Shen, Cheng, Li, Xu, Tao, and Zhou]{xu2024kdsurvey}
Xiaohan Xu, Ming Li, Chongyang Tao, Tao Shen, Reynold Cheng, Jinyang Li, Can Xu, Dacheng Tao, and Tianyi Zhou.
\newblock A survey on knowledge distillation of large language models.
\newblock \emph{arXiv preprint arXiv:2402.13116}, 2024.

\bibitem[Snell et~al.(2022)Snell, Klein, and Zhong]{snell2022learning}
Charlie Snell, Dan Klein, and Ruiqi Zhong.
\newblock Learning by distilling context.
\newblock \emph{arXiv preprint arXiv:2209.15189}, 2022.

\bibitem[Song and Zheng(2026)]{song2026survey}
Mingyang Song and Mao Zheng.
\newblock A survey of on-policy distillation for large language models.
\newblock \emph{arXiv preprint arXiv:2604.00626}, 2026.

\bibitem[Schulman et~al.(2017)Schulman, Wolski, Dhariwal, Radford, and Klimov]{schulman2017ppo}
John Schulman, Filip Wolski, Prafulla Dhariwal, Alec Radford, and Oleg Klimov.
\newblock Proximal policy optimization algorithms.
\newblock \emph{arXiv preprint arXiv:1707.06347}, 2017.

\bibitem[Christiano et~al.(2017)Christiano, Leike, Brown, Martic, Legg, and Amodei]{christiano2017deeprl}
Paul~F. Christiano, Jan Leike, Tom~B. Brown, Miljan Martic, Shane Legg, and Dario Amodei.
\newblock Deep reinforcement learning from human preferences.
\newblock In \emph{Advances in Neural Information Processing Systems (NeurIPS)}, 2017.

\bibitem[Ouyang et~al.(2022)Ouyang, Wu, Jiang, Almeida, Wainwright, Mishkin, Zhang, Agarwal, Slama, Ray, Schulman, Hilton, Kelton, Miller, Simens, Askell, Welinder, Christiano, Leike, and Lowe]{ouyang2022instructgpt}
Long Ouyang, Jeffrey Wu, Xu~Jiang, Diogo Almeida, Carroll~L. Wainwright, Pamela Mishkin, Chong Zhang, Sandhini Agarwal, Katarina Slama, Alex Ray, John Schulman, Jacob Hilton, Fraser Kelton, Luke Miller, Maddie Simens, Amanda Askell, Peter Welinder, Paul Christiano, Jan Leike, and Ryan Lowe.
\newblock Training language models to follow instructions with human feedback.
\newblock In \emph{Advances in Neural Information Processing Systems (NeurIPS)}, 2022.

\bibitem[Shao et~al.(2024)Shao, Wang, Zhu, Xu, Song, Bi, Zhang, Zhang, Li, Wu, and Guo]{shao2024deepseekmath}
Zhihong Shao, Peiyi Wang, Qihao Zhu, Runxin Xu, Junxiao Song, Xiao Bi, Haowei Zhang, Mingchuan Zhang, Y.~K. Li, Y.~Wu, and Daya Guo.
\newblock {DeepSeekMath}: Pushing the limits of mathematical reasoning in open language models.
\newblock \emph{arXiv preprint arXiv:2402.03300}, 2024.

\bibitem[Guo et~al.(2025)Guo, Yang, Zhang, Song, Wang, Zhu, Xu, Zhang, Ma, Bi, Zhang, Yu, Wu, Wu, Gou, Shao, Li, Gao, Liu, Xue, Wang, Wu, Feng, Lu, Zhao, Deng, Ruan, Dai, Chen, Ji, Li, Lin, Dai, Luo, Hao, Chen, Li, et~al.]{guo2025deepseekr1}
Daya Guo, Dejian Yang, Haowei Zhang, Junxiao Song, Peiyi Wang, Qihao Zhu, Runxin Xu, Ruoyu Zhang, Shirong Ma, Xiao Bi, Xiaokang Zhang, Xingkai Yu, Yu~Wu, Z.~F. Wu, Zhibin Gou, Zhihong Shao, Zhuoshu Li, Ziyi Gao, Aixin Liu, Bing Xue, Bingxuan Wang, Bochao Wu, Bei Feng, Chengda Lu, Chenggang Zhao, Chengqi Deng, Chong Ruan, Damai Dai, Deli Chen, Dongjie Ji, Erhang Li, Fangyun Lin, Fucong Dai, Fuli Luo, Guangbo Hao, Guanting Chen, Guowei Li, et~al.
\newblock {DeepSeek-R1} incentivizes reasoning in {LLM}s through reinforcement learning.
\newblock \emph{Nature}, 645\penalty0 (8081):\penalty0 633--638, 2025.
\newblock \doi{10.1038/s41586-025-09422-z}.
\newblock URL \url{https://doi.org/10.1038/s41586-025-09422-z}.

\bibitem[Yu et~al.(2025)Yu, Zhang, Zhu, Yuan, Zuo, Yue, Dai, Fan, Liu, Liu, Liu, Liu, Lin, Lin, Ma, Sheng, Tong, Zhang, Zhang, Zhang, Zhang, Zhu, Zhu, Chen, Chen, Wang, Yu, Song, Wei, Zhou, Liu, Ma, Zhang, Yan, Wu, and Wang]{yu2026dapo}
Qiying Yu, Zheng Zhang, Ruofei Zhu, Yufeng Yuan, Xiaochen Zuo, Yu~Yue, Weinan Dai, Tiantian Fan, Gaohong Liu, Juncai Liu, LingJun Liu, Xin Liu, Haibin Lin, Zhiqi Lin, Bole Ma, Guangming Sheng, Yuxuan Tong, Chi Zhang, Mofan Zhang, Ru~Zhang, Wang Zhang, Hang Zhu, Jinhua Zhu, Jiaze Chen, Jiangjie Chen, Chengyi Wang, Hongli Yu, Yuxuan Song, Xiangpeng Wei, Hao Zhou, Jingjing Liu, Wei-Ying Ma, Ya-Qin Zhang, Lin Yan, Yonghui Wu, and Mingxuan Wang.
\newblock {DAPO}: An open-source {LLM} reinforcement learning system at scale.
\newblock In \emph{Advances in Neural Information Processing Systems (NeurIPS)}, 2025.
\newblock URL \url{https://openreview.net/forum?id=2a36EMSSTp}.

\bibitem[Zhang et~al.(2025)Zhang, Zuo, He, Sun, Liu, Jiang, Fan, Tian, Jia, Li, Fu, Lv, Zhang, Zeng, Qu, Li, Wang, Wang, Long, Liu, Xu, Ma, Zhu, Hua, Liu, Li, Chen, Qu, Li, Chen, Yuan, Gao, Li, Ma, Cui, Liu, Qi, Ding, and Zhou]{zhang2025rlsurvey}
Kaiyan Zhang, Yuxin Zuo, Bingxiang He, Youbang Sun, Runze Liu, Che Jiang, Yuchen Fan, Kai Tian, Guoli Jia, Pengfei Li, Yu~Fu, Xingtai Lv, Yuchen Zhang, Sihang Zeng, Shang Qu, Haozhan Li, Shijie Wang, Yuru Wang, Xinwei Long, Fangfu Liu, Xiang Xu, Jiaze Ma, Xuekai Zhu, Ermo Hua, Yihao Liu, Zonglin Li, Huayu Chen, Xiaoye Qu, Yafu Li, Weize Chen, Zhenzhao Yuan, Junqi Gao, Dong Li, Zhiyuan Ma, Ganqu Cui, Zhiyuan Liu, Biqing Qi, Ning Ding, and Bowen Zhou.
\newblock A survey of reinforcement learning for large reasoning models.
\newblock \emph{arXiv preprint arXiv:2509.08827}, 2025.

\bibitem[Liu et~al.(2025)Liu, Yang, Qian, Yin, Wang, Li, Liu, Zhai, Liu, and Zhang]{liu2025rlllmsurvey}
Keliang Liu, Dingkang Yang, Ziyun Qian, Weijie Yin, Yuchi Wang, Hongsheng Li, Jun Liu, Peng Zhai, Yang Liu, and Lihua Zhang.
\newblock Reinforcement learning meets large language models: A survey of advancements and applications across the {LLM} lifecycle.
\newblock \emph{arXiv preprint arXiv:2509.16679}, 2025.

\bibitem[Chen et~al.(2025{\natexlab{a}})Chen, Qin, Liu, Peng, Guan, Wang, Hu, Zhou, Gao, and Che]{chen2025cotsurvey}
Qiguang Chen, Libo Qin, Jinhao Liu, Dengyun Peng, Jiannan Guan, Peng Wang, Mengkang Hu, Yuhang Zhou, Te~Gao, and Wanxiang Che.
\newblock Towards reasoning era: A survey of long chain-of-thought for reasoning large language models.
\newblock \emph{arXiv preprint arXiv:2503.09567}, 2025{\natexlab{a}}.

\bibitem[Zhao et~al.(2026)Zhao, Xie, Liu, Huang, Pang, Chen, and Grover]{zhao2026opsd}
Siyan Zhao, Zhihui Xie, Mengchen Liu, Jing Huang, Guan Pang, Feiyu Chen, and Aditya Grover.
\newblock Self-distilled reasoner: On-policy self-distillation for large language models.
\newblock \emph{arXiv preprint arXiv:2601.18734}, 2026.

\bibitem[Li et~al.(2026{\natexlab{a}})Li, Yang, Fang, Song, Zheng, Guo, Zhang, Wang, and Chua]{li2026srpo}
Gengsheng Li, Tianyu Yang, Junfeng Fang, Mingyang Song, Mao Zheng, Haiyun Guo, Dan Zhang, Jinqiao Wang, and Tat-Seng Chua.
\newblock Unifying group-relative and self-distillation policy optimization via sample routing.
\newblock \emph{arXiv preprint arXiv:2604.02288}, 2026{\natexlab{a}}.

\bibitem[H{\"u}botter et~al.(2026)H{\"u}botter, L{\"u}beck, Behric, Baumann, Bagatella, Marta, Hakimi, Shenfeld, Kleine~Buening, Guestrin, and Krause]{hubotter2026sdpo}
Jonas H{\"u}botter, Frederike L{\"u}beck, Lejs Behric, Anton Baumann, Marco Bagatella, Daniel Marta, Ido Hakimi, Idan Shenfeld, Thomas Kleine~Buening, Carlos Guestrin, and Andreas Krause.
\newblock Reinforcement learning via self-distillation.
\newblock \emph{arXiv preprint arXiv:2601.20802}, 2026.

\bibitem[Yang et~al.(2026)Yang, Qin, Si, Chen, Gu, Yao, Lin, Wang, Wang, and Duan]{yang2026rlsd}
Chenxu Yang, Chuanyu Qin, Qingyi Si, Minghui Chen, Naibin Gu, Dingyu Yao, Zheng Lin, Weiping Wang, Jiaqi Wang, and Nan Duan.
\newblock Self-distilled {RLVR}.
\newblock \emph{arXiv preprint arXiv:2604.03128}, 2026.

\bibitem[Shenfeld et~al.(2026)Shenfeld, Damani, H{\"u}botter, and Agrawal]{shenfeld2026selfdistillation}
Idan Shenfeld, Mehul Damani, Jonas H{\"u}botter, and Pulkit Agrawal.
\newblock Self-distillation enables continual learning.
\newblock In \emph{ICLR 2026 Workshop on Lifelong Agents: Learning, Aligning, Evolving}, 2026.
\newblock URL \url{https://openreview.net/forum?id=HlWA3V6iKF}.

\bibitem[Kim et~al.(2026)Kim, Luo, Kim, Lee, Kim, Jeon, Li, and Yang]{kim2026whyopsd}
Jeonghye Kim, Xufang Luo, Minbeom Kim, Sangmook Lee, Dohyung Kim, Jiwon Jeon, Dongsheng Li, and Yuqing Yang.
\newblock Why does self-distillation (sometimes) degrade the reasoning capability of {LLM}s?
\newblock \emph{arXiv preprint arXiv:2603.24472}, 2026.

\bibitem[Agarwal et~al.(2024)Agarwal, Vieillard, Zhou, Stanczyk, Ramos, Geist, and Bachem]{agarwal2024onpolicy}
Rishabh Agarwal, Nino Vieillard, Yongchao Zhou, Piotr Stanczyk, Sabela Ramos, Matthieu Geist, and Olivier Bachem.
\newblock On-policy distillation of language models: Learning from self-generated mistakes.
\newblock In \emph{International Conference on Learning Representations (ICLR)}, 2024.

\bibitem[Zheng et~al.(2025)Zheng, Zhang, Zhang, Lin, Lu, Yu, Liu, Zhou, and Lin]{zheng2024processbench}
Chujie Zheng, Zhenru Zhang, Beichen Zhang, Runji Lin, Keming Lu, Bowen Yu, Dayiheng Liu, Jingren Zhou, and Junyang Lin.
\newblock {ProcessBench}: Identifying process errors in mathematical reasoning.
\newblock In \emph{Proceedings of the 63rd Annual Meeting of the Association for Computational Linguistics (ACL)}, 2025.

\bibitem[Wang et~al.(2026{\natexlab{a}})Wang, Yu, Gao, Zheng, Liu, Lu, Dang, Chen, Yang, Zhang, Liu, Yang, Zhao, Yue, Song, Yu, Huang, and Lin]{wang2025highentropy}
Shenzhi Wang, Le~Yu, Chang Gao, Chujie Zheng, Shixuan Liu, Rui Lu, Kai Dang, Xiong-Hui Chen, Jianxin Yang, Zhenru Zhang, Yuqiong Liu, An~Yang, Andrew Zhao, Yang Yue, Shiji Song, Bowen Yu, Gao Huang, and Junyang Lin.
\newblock Beyond the 80/20 rule: High-entropy minority tokens drive effective reinforcement learning for {LLM} reasoning.
\newblock In \emph{The Thirty-ninth Annual Conference on Neural Information Processing Systems}, 2026{\natexlab{a}}.
\newblock URL \url{https://openreview.net/forum?id=yfcpdY4gMP}.

\bibitem[Jin et~al.(2026)Jin, Min, Yang, Kadhe, Zhou, Wei, Baracaldo, and Lee]{jin2026entropy}
Woogyeol Jin, Taywon Min, Yongjin Yang, Swanand~Ravindra Kadhe, Yi~Zhou, Dennis Wei, Nathalie Baracaldo, and Kimin Lee.
\newblock Entropy-aware on-policy distillation of language models.
\newblock \emph{arXiv preprint arXiv:2603.07079}, 2026.

\bibitem[Yang et~al.(2025)Yang, Yang, Hui, et~al.]{yang2025qwen3}
An~Yang, Baosong Yang, Binyuan Hui, et~al.
\newblock {Qwen3} technical report.
\newblock \emph{arXiv preprint arXiv:2505.09388}, 2025.

\bibitem[Hendrycks et~al.(2021{\natexlab{a}})Hendrycks, Burns, Kadavath, Arora, Basart, Tang, Song, and Steinhardt]{hendrycks2021math}
Dan Hendrycks, Collin Burns, Saurav Kadavath, Akul Arora, Steven Basart, Eric Tang, Dawn Song, and Jacob Steinhardt.
\newblock Measuring mathematical problem solving with the {MATH} dataset.
\newblock In \emph{Advances in Neural Information Processing Systems (NeurIPS) Datasets and Benchmarks Track}, 2021{\natexlab{a}}.

\bibitem[Lewkowycz et~al.(2022)Lewkowycz, Andreassen, Dohan, Dyer, Michalewski, Ramasesh, Slone, Anil, Schlag, Gutman-Solo, Wu, Neyshabur, Gur-Ari, and Misra]{lewkowycz2022solving}
Aitor Lewkowycz, Anders Andreassen, David Dohan, Ethan Dyer, Henryk Michalewski, Vinay Ramasesh, Ambrose Slone, Cem Anil, Imanol Schlag, Theo Gutman-Solo, Yuhuai Wu, Behnam Neyshabur, Guy Gur-Ari, and Vedant Misra.
\newblock Solving quantitative reasoning problems with language models.
\newblock In \emph{Advances in Neural Information Processing Systems (NeurIPS)}, 2022.

\bibitem[He et~al.(2024)He, Luo, Bai, Hu, Thai, Shen, Hu, Han, Huang, Zhang, Liu, Qi, Liu, and Sun]{he2024olympiadbench}
Chaoqun He, Renjie Luo, Yuzhuo Bai, Shengding Hu, Zhen~Leng Thai, Junhao Shen, Jinyi Hu, Xu~Han, Yujie Huang, Yuxiang Zhang, Jie Liu, Lei Qi, Zhiyuan Liu, and Maosong Sun.
\newblock {OlympiadBench}: A challenging benchmark for promoting {AGI} with olympiad-level bilingual multimodal scientific problems.
\newblock In \emph{Proceedings of the 62nd Annual Meeting of the Association for Computational Linguistics (ACL)}, 2024.

\bibitem[Sheng et~al.(2025)Sheng, Zhang, Ye, Wu, Zhang, Zhang, Peng, Lin, and Wu]{sheng2024hybridflow}
Guangming Sheng, Chi Zhang, Zilingfeng Ye, Xibin Wu, Wang Zhang, Ru~Zhang, Yanghua Peng, Haibin Lin, and Chuan Wu.
\newblock Hybridflow: A flexible and efficient rlhf framework.
\newblock In \emph{Proceedings of the Twentieth European Conference on Computer Systems}, EuroSys '25, page 1279–1297, New York, NY, USA, 2025. Association for Computing Machinery.
\newblock ISBN 9798400711961.
\newblock \doi{10.1145/3689031.3696075}.
\newblock URL \url{https://doi.org/10.1145/3689031.3696075}.

\bibitem[Cheng et~al.(2026)Cheng, Huang, Zhu, Dai, Zhao, Zhang, and Wei]{DBLP:conf/aaai/ChengHZDZZW26}
Daixuan Cheng, Shaohan Huang, Xuekai Zhu, Bo~Dai, Xin Zhao, Zhenliang Zhang, and Furu Wei.
\newblock Reasoning with exploration: An entropy perspective.
\newblock In Sven Koenig, Chad Jenkins, and Matthew~E. Taylor, editors, \emph{Fortieth {AAAI} Conference on Artificial Intelligence, Thirty-Eighth Conference on Innovative Applications of Artificial Intelligence, Sixteenth Symposium on Educational Advances in Artificial Intelligence, {AAAI} 2026, Singapore, January 20-27, 2026}, pages 30377--30385. {AAAI} Press, 2026.
\newblock \doi{10.1609/AAAI.V40I36.40290}.
\newblock URL \url{https://doi.org/10.1609/aaai.v40i36.40290}.

\bibitem[He et~al.(2026{\natexlab{a}})He, Wu, Liu, Ge, Zhou, Wu, Zheng, Lin, Zhong, and Zhang]{he2026rethinkingtokenlevelcreditassignment}
Yuhang He, Haodong Wu, Siyi Liu, Hongyu Ge, Hange Zhou, Keyi Wu, Zhuo Zheng, Qihong Lin, Zixin Zhong, and Yongqi Zhang.
\newblock Rethinking token-level credit assignment in rlvr: A polarity-entropy analysis, 2026{\natexlab{a}}.
\newblock URL \url{https://arxiv.org/abs/2604.11056}.

\bibitem[Huang et~al.(2026)Huang, Meng, Wu, Lu, Ma, Chen, Wang, Ding, Wu, Wang, He, Wang, and Zhou]{huang2026directionrlvrupdatesllm}
Kexin Huang, Haoming Meng, Junkang Wu, Jinda Lu, Chiyu Ma, Ziqian Chen, Xue Wang, Bolin Ding, Jiancan Wu, Xiang Wang, Xiangnan He, Guoyin Wang, and Jingren Zhou.
\newblock On the direction of rlvr updates for llm reasoning: Identification and exploitation, 2026.
\newblock URL \url{https://arxiv.org/abs/2603.22117}.

\bibitem[Liu et~al.(2026)Liu, He, Zhan, Tao, Zheng, Wu, Wang, Guan, Sheng, Zhang, Li, Duan, and Li]{liu2026stapostabilizingreinforcementlearning}
Shiqi Liu, Zeyu He, Guojian Zhan, Letian Tao, Zhilong Zheng, Jiang Wu, Yinuo Wang, Yang Guan, Kehua Sheng, Bo~Zhang, Keqiang Li, Jingliang Duan, and Shengbo~Eben Li.
\newblock Stapo: Stabilizing reinforcement learning for llms by silencing rare spurious tokens, 2026.
\newblock URL \url{https://arxiv.org/abs/2602.15620}.

\bibitem[Chen et~al.(2025{\natexlab{b}})Chen, Li, Sun, and Yu]{chen2025highentropyexplorationcorrectnessawarelowentropy}
Xinzhu Chen, Xuesheng Li, Zhongxiang Sun, and Weijie Yu.
\newblock Beyond high-entropy exploration: Correctness-aware low-entropy segment-based advantage shaping for reasoning llms, 2025{\natexlab{b}}.
\newblock URL \url{https://arxiv.org/abs/2512.00908}.

\bibitem[Penaloza et~al.(2026)Penaloza, Vattikonda, Gontier, Lacoste, Charlin, and Caccia]{penaloza2026pidistill}
Emiliano Penaloza, Dheeraj Vattikonda, Nicolas Gontier, Alexandre Lacoste, Laurent Charlin, and Massimo Caccia.
\newblock Privileged information distillation for language models.
\newblock \emph{arXiv preprint arXiv:2602.04942}, 2026.

\bibitem[Stein et~al.(2026)Stein, Huang, and Goldstein]{stein2026gates}
Alex Stein, Furong Huang, and Tom Goldstein.
\newblock {GATES}: Self-distillation under privileged context with consensus gating.
\newblock \emph{arXiv preprint arXiv:2602.20574}, 2026.

\bibitem[Zhang et~al.(2026{\natexlab{a}})Zhang, Zhu, Chong, Tu, et~al.]{zhang2026piplay}
Yaocheng Zhang, Yuanheng Zhu, Wenyue Chong, Songjun Tu, et~al.
\newblock {$\pi$-Play}: Multi-agent self-play via privileged self-distillation without external data.
\newblock \emph{arXiv preprint arXiv:2604.14054}, 2026{\natexlab{a}}.

\bibitem[Fu et~al.(2026)Fu, Huang, Jiang, Zhu, and Zhao]{fu2026revisiting}
Yuqian Fu, Haohuan Huang, Kaiwen Jiang, Yuanheng Zhu, and Dongbin Zhao.
\newblock Revisiting on-policy distillation: Empirical failure modes and simple fixes.
\newblock \emph{arXiv preprint arXiv:2603.25562}, 2026.

\bibitem[Liu et~al.(2023)Liu, Xia, Wang, and Zhang]{liu2023evalplus}
Jiawei Liu, Chunqiu~Steven Xia, Yuyao Wang, and Lingming Zhang.
\newblock Is your code generated by {ChatGPT} really correct? rigorous evaluation of large language models for code generation.
\newblock In \emph{Advances in Neural Information Processing Systems (NeurIPS)}, 2023.

\bibitem[Rein et~al.(2024)Rein, Hou, Stickland, Petty, Pang, Dirani, Michael, and Bowman]{rein2023gpqa}
David Rein, Betty~Li Hou, Asa~Cooper Stickland, Jackson Petty, Richard~Yuanzhe Pang, Julien Dirani, Julian Michael, and Samuel~R. Bowman.
\newblock {GPQA}: A graduate-level google-proof {Q\&A} benchmark.
\newblock In \emph{Conference on Language Modeling (COLM)}, 2024.

\bibitem[Hendrycks et~al.(2021{\natexlab{b}})Hendrycks, Burns, Basart, Zou, Mazeika, Song, and Steinhardt]{hendrycks2021mmlu}
Dan Hendrycks, Collin Burns, Steven Basart, Andy Zou, Mantas Mazeika, Dawn Song, and Jacob Steinhardt.
\newblock Measuring massive multitask language understanding.
\newblock In \emph{International Conference on Learning Representations (ICLR)}, 2021{\natexlab{b}}.

\bibitem[Patil et~al.(2025)Patil, Mao, Yan, Ji, Suresh, Stoica, and Gonzalez]{patil2025bfcl}
Shishir~G Patil, Huanzhi Mao, Fanjia Yan, Charlie Cheng-Jie Ji, Vishnu Suresh, Ion Stoica, and Joseph~E. Gonzalez.
\newblock The berkeley function calling leaderboard ({BFCL}): From tool use to agentic evaluation of large language models.
\newblock In \emph{Forty-second International Conference on Machine Learning}, 2025.
\newblock URL \url{https://openreview.net/forum?id=2GmDdhBdDk}.

\bibitem[Olmo et~al.(2026)Olmo, :, Ettinger, Bertsch, Kuehl, Graham, Heineman, Groeneveld, Brahman, Timbers, Ivison, Morrison, Poznanski, Lo, Soldaini, Jordan, Chen, Noukhovitch, Lambert, Walsh, Dasigi, Berry, Malik, Shah, Geng, Arora, Gupta, Anderson, Xiao, Murray, Romero, Graf, Asai, Bhagia, Wettig, Liu, Rangapur, Anastasiades, Huang, Schwenk, Trivedi, Magnusson, Lochner, Liu, Miranda, Sap, Morgan, Schmitz, Guerquin, Wilson, Huff, Bras, Xin, Shao, Skjonsberg, Shen, Li, Wilde, Pyatkin, Merrill, Chang, Gu, Zeng, Sabharwal, Zettlemoyer, Koh, Farhadi, Smith, and Hajishirzi]{olmo2026}
Team Olmo, :, Allyson Ettinger, Amanda Bertsch, Bailey Kuehl, David Graham, David Heineman, Dirk Groeneveld, Faeze Brahman, Finbarr Timbers, Hamish Ivison, Jacob Morrison, Jake Poznanski, Kyle Lo, Luca Soldaini, Matt Jordan, Mayee Chen, Michael Noukhovitch, Nathan Lambert, Pete Walsh, Pradeep Dasigi, Robert Berry, Saumya Malik, Saurabh Shah, Scott Geng, Shane Arora, Shashank Gupta, Taira Anderson, Teng Xiao, Tyler Murray, Tyler Romero, Victoria Graf, Akari Asai, Akshita Bhagia, Alexander Wettig, Alisa Liu, Aman Rangapur, Chloe Anastasiades, Costa Huang, Dustin Schwenk, Harsh Trivedi, Ian Magnusson, Jaron Lochner, Jiacheng Liu, Lester James~V. Miranda, Maarten Sap, Malia Morgan, Michael Schmitz, Michal Guerquin, Michael Wilson, Regan Huff, Ronan~Le Bras, Rui Xin, Rulin Shao, Sam Skjonsberg, Shannon~Zejiang Shen, Shuyue~Stella Li, Tucker Wilde, Valentina Pyatkin, Will Merrill, Yapei Chang, Yuling Gu, Zhiyuan Zeng, Ashish Sabharwal, Luke Zettlemoyer, Pang~Wei Koh, Ali Farhadi, Noah~A. Smith, and Hannaneh
  Hajishirzi.
\newblock Olmo 3, 2026.
\newblock URL \url{https://arxiv.org/abs/2512.13961}.

\bibitem[Grattafiori et~al.(2024)Grattafiori, Dubey, Jauhri, Pandey, Kadian, Al-Dahle, Letman, Mathur, Schelten, Vaughan, et~al.]{grattafiori2024llama3}
Aaron Grattafiori, Abhimanyu Dubey, Abhinav Jauhri, Abhinav Pandey, Abhishek Kadian, Ahmad Al-Dahle, Aiesha Letman, Akhil Mathur, Alan Schelten, Alex Vaughan, et~al.
\newblock The {Llama 3} herd of models.
\newblock \emph{arXiv preprint arXiv:2407.21783}, 2024.

\bibitem[Schulman et~al.(2016)Schulman, Moritz, Levine, Jordan, and Abbeel]{schulman2016gae}
John Schulman, Philipp Moritz, Sergey Levine, Michael~I. Jordan, and Pieter Abbeel.
\newblock High-dimensional continuous control using generalized advantage estimation.
\newblock In \emph{International Conference on Learning Representations (ICLR)}, 2016.

\bibitem[Cobbe et~al.(2021)Cobbe, Kosaraju, Bavarian, Chen, Jun, Kaiser, Plappert, Tworek, Hilton, Nakano, Hesse, and Schulman]{cobbe2021gsm8k}
Karl Cobbe, Vineet Kosaraju, Mohammad Bavarian, Mark Chen, Heewoo Jun, Lukasz Kaiser, Matthias Plappert, Jerry Tworek, Jacob Hilton, Reiichiro Nakano, Christopher Hesse, and John Schulman.
\newblock Training verifiers to solve math word problems.
\newblock \emph{arXiv preprint arXiv:2110.14168}, 2021.

\bibitem[Uesato et~al.(2022)Uesato, Kushman, Kumar, Song, Siegel, Wang, Creswell, Irving, and Higgins]{uesato2022solving}
Jonathan Uesato, Nate Kushman, Ramana Kumar, Francis Song, Noah Siegel, Lisa Wang, Antonia Creswell, Geoffrey Irving, and Irina Higgins.
\newblock Solving math word problems with process- and outcome-based feedback.
\newblock \emph{arXiv preprint arXiv:2211.14275}, 2022.

\bibitem[Lightman et~al.(2024)Lightman, Kosaraju, Burda, Edwards, Baker, Lee, Leike, Schulman, Sutskever, and Cobbe]{lightman2024letsverify}
Hunter Lightman, Vineet Kosaraju, Yura Burda, Harri Edwards, Bowen Baker, Teddy Lee, Jan Leike, John Schulman, Ilya Sutskever, and Karl Cobbe.
\newblock Let's verify step by step.
\newblock In \emph{International Conference on Learning Representations (ICLR)}, 2024.

\bibitem[Wang et~al.(2024)Wang, Li, Shao, Xu, Dai, Li, Chen, Wu, and Sui]{wang2024mathshepherd}
Peiyi Wang, Lei Li, Zhihong Shao, Runxin Xu, Damai Dai, Yifei Li, Deli Chen, Yu~Wu, and Zhifang Sui.
\newblock {Math-Shepherd}: Verify and reinforce {LLM}s step-by-step without human annotations.
\newblock In \emph{Proceedings of the 62nd Annual Meeting of the Association for Computational Linguistics (ACL)}, 2024.

\bibitem[Li et~al.(2026{\natexlab{b}})Li, Zuo, He, Zhang, Xiao, Qian, Yu, Gao, Yang, Liu, and Ding]{li2026rethinking}
Yaxuan Li, Yuxin Zuo, Bingxiang He, Jinqian Zhang, Chaojun Xiao, Cheng Qian, Tianyu Yu, Huan-ang Gao, Wenkai Yang, Zhiyuan Liu, and Ning Ding.
\newblock Rethinking on-policy distillation of large language models: Phenomenology, mechanism, and recipe.
\newblock \emph{arXiv preprint arXiv:2604.13016}, 2026{\natexlab{b}}.

\bibitem[Ye et~al.(2026{\natexlab{a}})Ye, Dong, Wu, Huang, and Wei]{ye2026opcd}
Tianzhu Ye, Li~Dong, Xun Wu, Shaohan Huang, and Furu Wei.
\newblock On-policy context distillation for language models.
\newblock \emph{arXiv preprint arXiv:2602.12275}, 2026{\natexlab{a}}.

\bibitem[Ye et~al.(2026{\natexlab{b}})Ye, Dong, Dong, Wu, Huang, and Wei]{ye2026oel}
Tianzhu Ye, Li~Dong, Qingxiu Dong, Xun Wu, Shaohan Huang, and Furu Wei.
\newblock Online experiential learning for language models.
\newblock \emph{arXiv preprint arXiv:2603.16856}, 2026{\natexlab{b}}.

\bibitem[Ye et~al.(2025)Ye, Dong, Chi, Wu, Huang, and Wei]{ye2025gad}
Tianzhu Ye, Li~Dong, Zewen Chi, Xun Wu, Shaohan Huang, and Furu Wei.
\newblock Black-box on-policy distillation of large language models.
\newblock \emph{arXiv preprint arXiv:2511.10643}, 2025.

\bibitem[Chen et~al.(2026)Chen, Wang, Zhu, Qiu, Dong, Deng, Sang, Wang, Geramifard, and Luo]{chen2026soda}
Xiwen Chen, Jingjing Wang, Wenhui Zhu, Peijie Qiu, Xuanzhao Dong, Yueyue Deng, Hejian Sang, Zhipeng Wang, Alborz Geramifard, and Feng Luo.
\newblock {SODA}: Semi on-policy black-box distillation for large language models.
\newblock \emph{arXiv preprint arXiv:2604.03873}, 2026.

\bibitem[Wang et~al.(2026{\natexlab{b}})Wang, Liu, Chen, Hu, Zhang, Cao, Wang, Yang, Xie, and Chen]{wang2026madopd}
Jianze Wang, Ying Liu, Jinlong Chen, Xuchun Hu, Qilong Zhang, Yu~Cao, Jun Wang, Hua Yang, Yong Xie, and Qianglong Chen.
\newblock {MAD-OPD}: Breaking the ceiling in on-policy distillation via multi-agent debate.
\newblock \emph{arXiv preprint arXiv:2605.01347}, 2026{\natexlab{b}}.

\bibitem[Sang et~al.(2026)Sang, Xu, Zhou, He, Wang, and Sun]{sang2026crisp}
Hejian Sang, Yuanda Xu, Zhengze Zhou, Ran He, Zhipeng Wang, and Jiachen Sun.
\newblock Crisp: Compressed reasoning via iterative self-policy distillation.
\newblock \emph{arXiv preprint arXiv:2603.05433}, 2026.

\bibitem[He et~al.(2026{\natexlab{b}})He, Kaur, Bhaskar, Yang, Liu, Ri, Fowl, Panigrahi, Chen, and Arora]{he2026sdzero}
Yinghui He, Simran Kaur, Adithya Bhaskar, Yongjin Yang, Jiarui Liu, Narutatsu Ri, Liam Fowl, Abhishek Panigrahi, Danqi Chen, and Sanjeev Arora.
\newblock Self-distillation zero: Self-revision turns binary rewards into dense supervision.
\newblock \emph{arXiv preprint arXiv:2604.12002}, 2026{\natexlab{b}}.

\bibitem[Zhang et~al.(2026{\natexlab{b}})Zhang, Ding, Pan, Yang, Kang, Xiong, and Gu]{zhang2026opsdl}
Xinsen Zhang, Zhenkai Ding, Tianjun Pan, Run Yang, Chun Kang, Xue Xiong, and Jingnan Gu.
\newblock Opsdl: On-policy self-distillation for long-context language models, 2026{\natexlab{b}}.
\newblock URL \url{https://arxiv.org/abs/2604.17535}.

\bibitem[Zhang et~al.(2026{\natexlab{c}})Zhang, Bai, Zheng, Jaitly, Collobert, and Zhang]{zhang2026ssd}
Ruixiang Zhang, Richard~He Bai, Huangjie Zheng, Navdeep Jaitly, Ronan Collobert, and Yizhe Zhang.
\newblock Embarrassingly simple self-distillation improves code generation.
\newblock \emph{arXiv preprint arXiv:2604.01193}, 2026{\natexlab{c}}.

\bibitem[Wei et~al.(2022{\natexlab{a}})Wei, Wang, Schuurmans, Bosma, Ichter, Xia, Chi, Le, and Zhou]{wei2022cot}
Jason Wei, Xuezhi Wang, Dale Schuurmans, Maarten Bosma, Brian Ichter, Fei Xia, Ed~H. Chi, Quoc~V. Le, and Denny Zhou.
\newblock Chain-of-thought prompting elicits reasoning in large language models.
\newblock In \emph{Advances in Neural Information Processing Systems (NeurIPS)}, 2022{\natexlab{a}}.

\bibitem[Kojima et~al.(2022)Kojima, Gu, Reid, Matsuo, and Iwasawa]{kojima2022zeroshot}
Takeshi Kojima, Shixiang~Shane Gu, Machel Reid, Yutaka Matsuo, and Yusuke Iwasawa.
\newblock Large language models are zero-shot reasoners.
\newblock In \emph{Advances in Neural Information Processing Systems (NeurIPS)}, 2022.

\bibitem[Wei et~al.(2022{\natexlab{b}})Wei, Tay, Bommasani, Raffel, Zoph, Borgeaud, Yogatama, Bosma, Zhou, Metzler, Chi, Hashimoto, Vinyals, Liang, Dean, and Fedus]{wei2022emergent}
Jason Wei, Yi~Tay, Rishi Bommasani, Colin Raffel, Barret Zoph, Sebastian Borgeaud, Dani Yogatama, Maarten Bosma, Denny Zhou, Donald Metzler, Ed~H. Chi, Tatsunori Hashimoto, Oriol Vinyals, Percy Liang, Jeff Dean, and William Fedus.
\newblock Emergent abilities of large language models.
\newblock \emph{Transactions on Machine Learning Research (TMLR)}, 2022{\natexlab{b}}.

\bibitem[Wang et~al.(2023)Wang, Wei, Schuurmans, Le, Chi, Narang, Chowdhery, and Zhou]{wang2023selfconsistency}
Xuezhi Wang, Jason Wei, Dale Schuurmans, Quoc~V. Le, Ed~H. Chi, Sharan Narang, Aakanksha Chowdhery, and Denny Zhou.
\newblock Self-consistency improves chain of thought reasoning in language models.
\newblock In \emph{International Conference on Learning Representations (ICLR)}, 2023.

\bibitem[Yao et~al.(2023)Yao, Yu, Zhao, Shafran, Griffiths, Cao, and Narasimhan]{yao2023tot}
Shunyu Yao, Dian Yu, Jeffrey Zhao, Izhak Shafran, Thomas~L. Griffiths, Yuan Cao, and Karthik Narasimhan.
\newblock Tree of thoughts: Deliberate problem solving with large language models.
\newblock In \emph{Advances in Neural Information Processing Systems (NeurIPS)}, 2023.

\bibitem[Zelikman et~al.(2022)Zelikman, Wu, Mu, and Goodman]{zelikman2022star}
Eric Zelikman, Yuhuai Wu, Jesse Mu, and Noah~D. Goodman.
\newblock {STaR}: Bootstrapping reasoning with reasoning.
\newblock In \emph{Advances in Neural Information Processing Systems (NeurIPS)}, 2022.

\bibitem[Singh et~al.(2024)Singh, Co-Reyes, Agarwal, Anand, Patil, Garcia, Liu, Harrison, Lee, Xu, Parisi, Kumar, Alemi, Rizkowsky, Nova, Adlam, Bohnet, Elsayed, Sedghi, Mordatch, Simpson, Gur, Snoek, Pennington, Hron, Kenealy, Swersky, Mahajan, Culp, Xiao, Bileschi, Constant, Novak, Liu, Warkentin, Qian, Bansal, Dyer, Neyshabur, Sohl-Dickstein, and Fiedel]{singh2024beyond}
Avi Singh, John~D. Co-Reyes, Rishabh Agarwal, Ankesh Anand, Piyush Patil, Xavier Garcia, Peter~J. Liu, James Harrison, Jaehoon Lee, Kelvin Xu, Aaron Parisi, Abhishek Kumar, Alex Alemi, Alex Rizkowsky, Azade Nova, Ben Adlam, Bernd Bohnet, Gamaleldin Elsayed, Hanie Sedghi, Igor Mordatch, Isabelle Simpson, Izzeddin Gur, Jasper Snoek, Jeffrey Pennington, Jiri Hron, Kathleen Kenealy, Kevin Swersky, Kshiteej Mahajan, Laura Culp, Lechao Xiao, Maxwell~L. Bileschi, Noah Constant, Roman Novak, Rosanne Liu, Tris Warkentin, Yundi Qian, Yamini Bansal, Ethan Dyer, Behnam Neyshabur, Jascha Sohl-Dickstein, and Noah Fiedel.
\newblock Beyond human data: Scaling self-training for problem-solving with language models.
\newblock \emph{Transactions on Machine Learning Research (TMLR)}, 2024.

\bibitem[Madaan et~al.(2023)Madaan, Tandon, Gupta, Hallinan, Gao, Wiegreffe, Alon, Dziri, Prabhumoye, Yang, Gupta, Majumder, Hermann, Welleck, Yazdanbakhsh, and Clark]{madaan2023selfrefine}
Aman Madaan, Niket Tandon, Prakhar Gupta, Skyler Hallinan, Luyu Gao, Sarah Wiegreffe, Uri Alon, Nouha Dziri, Shrimai Prabhumoye, Yiming Yang, Shashank Gupta, Bodhisattwa~Prasad Majumder, Katherine Hermann, Sean Welleck, Amir Yazdanbakhsh, and Peter Clark.
\newblock {Self-Refine}: Iterative refinement with self-feedback.
\newblock In \emph{Advances in Neural Information Processing Systems (NeurIPS)}, 2023.

\bibitem[Huang et~al.(2024)Huang, Chen, Mishra, Zheng, Yu, Song, and Zhou]{huang2024selfcorrect}
Jie Huang, Xinyun Chen, Swaroop Mishra, Huaixiu~Steven Zheng, Adams~Wei Yu, Xinying Song, and Denny Zhou.
\newblock Large language models cannot self-correct reasoning yet.
\newblock In \emph{International Conference on Learning Representations (ICLR)}, 2024.

\bibitem[Yu et~al.(2024)Yu, Jiang, Shi, Yu, Liu, Zhang, Kwok, Li, Weller, and Liu]{yu2024metamath}
Longhui Yu, Weisen Jiang, Han Shi, Jincheng Yu, Zhengying Liu, Yu~Zhang, James~T. Kwok, Zhenguo Li, Adrian Weller, and Weiyang Liu.
\newblock {MetaMath}: Bootstrap your own mathematical questions for large language models.
\newblock In \emph{International Conference on Learning Representations (ICLR)}, 2024.

\bibitem[Luo et~al.(2024)Luo, Sun, Xu, Zhao, Lou, Tao, Geng, Lin, Chen, Tang, and Zhang]{luo2024wizardmath}
Haipeng Luo, Qingfeng Sun, Can Xu, Pu~Zhao, Jianguang Lou, Chongyang Tao, Xiubo Geng, Qingwei Lin, Shifeng Chen, Yansong Tang, and Dongmei Zhang.
\newblock {WizardMath}: Empowering mathematical reasoning for large language models via reinforced evol-instruct.
\newblock \emph{arXiv preprint arXiv:2308.09583}, 2024.

\bibitem[Gou et~al.(2024)Gou, Shao, Gong, Shen, Yang, Huang, Duan, and Chen]{gou2024tora}
Zhibin Gou, Zhihong Shao, Yeyun Gong, Yelong Shen, Yujiu Yang, Minlie Huang, Nan Duan, and Weizhu Chen.
\newblock {ToRA}: A tool-integrated reasoning agent for mathematical problem solving.
\newblock In \emph{International Conference on Learning Representations (ICLR)}, 2024.

\end{thebibliography}
}

\clearpage


\appendix


\section*{Appendix Overview and Roadmap}
\phantomsection
\addcontentsline{toc}{section}{Appendix Overview and Roadmap}

The appendix is organized in the same order in which the main text first calls for supporting material. It starts from the full diagnostic evidence behind Section~\ref{sec:prelim-analysis}, then gives the mathematical justification and algorithmic form of DASD, then expands the experimental protocol and results, and finally provides additional analyses and related work. The intended reading path is therefore: \emph{why uniform OPSD fails} $\rightarrow$ \emph{why the signed objective is valid} $\rightarrow$ \emph{how the method is implemented} $\rightarrow$ \emph{what was measured in the main, cross-domain, and cross-family settings}.

\begin{tcolorbox}[colback=blue!3,colframe=blue!45!black,boxrule=0.45pt,arc=1mm,left=5pt,right=5pt,top=4pt,bottom=4pt]
\textbf{Appendix table of contents.}
\begin{enumerate}[leftmargin=1.4em,itemsep=1pt,topsep=2pt]
  \item \hyperref[app:sec3_detailed]{Appendix~\ref*{app:sec3_detailed}} expands the three diagnostic probes that motivate entropy-conditioned teacher direction.
  \item \hyperref[app:derivation]{Appendix~\ref*{app:derivation}} proves Proposition~\ref{prop:dasd-credit} and records the approximation used by the sampled token estimator.
  \item \hyperref[app:algorithm]{Appendix~\ref*{app:algorithm}} gives the full DASD pseudocode aligned with Section~\ref{sec:method}.
  \item \hyperref[app:full_results]{Appendix~\ref*{app:full_results}} reports extended Avg@16 benchmark and Pass@16 values.
  \item \hyperref[app:cross_domain_family]{Appendix~\ref*{app:cross_domain_family}} separates the Qwen3-only cross-domain check from the math-only cross-family check across Qwen3, Olmo, and Llama.
  \item \hyperref[app:exp_details]{Appendix~\ref*{app:exp_details}} details datasets, baselines, metrics, sampling, and implementation settings.
  \item \hyperref[app:additional_analysis]{Appendix~\ref*{app:additional_analysis}} expands the design-space ablation; its first subsection, \hyperref[app:tau_sensitivity]{Appendix~\ref*{app:tau_sensitivity}}, analyzes entropy-quantile sensitivity. It also includes appendix-only figures, including a model-scale analysis showing that the DASD$-$GRPO gap widens with Qwen3 scale (Figure~\ref{fig:model_scaling}).
  \item \hyperref[app:related_work]{Appendix~\ref*{app:related_work}} gives the longer positioning against RLVR, OPD, OPSD, RLSD, SDPO, SRPO, and entropy-aware credit methods.
  \item \hyperref[app:limitations_discussion]{Appendix~\ref*{app:limitations_discussion}} states the limitations DASD inherits from the OPSD setup and discusses extensions of the direction-adaptive credit view.
\end{enumerate}
\end{tcolorbox}


\section{Detailed Probe Exposition for Section~\ref{sec:prelim-analysis}}\label{app:sec3_detailed}

This section gives the full reasoning behind the three compressed probes in Section~\ref{sec:prelim-analysis}. The probes are deliberately diagnostic rather than method-tuning experiments: they isolate the local direction of privileged self-teacher pressure before any DASD-specific design choice is introduced. The central question is whether OPSD fails because self-distillation is inherently harmful, or because one teacher-pressure direction is applied uniformly to token populations that play different roles in mathematical reasoning.

\subsection{Probe 1: Does reversing teacher pressure restore reasoning?}\label{app:probe1_details}

\noindent\textbf{Question and setup.}
Prior analysis of OPSD argues that conditioning the self-teacher on privileged information $r$ collapses the teacher's residual uncertainty and therefore trains the student into a solution-conditioned style. If the harmful part were only the attraction toward that over-confident teacher, the obvious mirror repair would be to reverse the sign and push the student away from the teacher. We test this mirror hypothesis by introducing a global teacher-pressure sign $s_t\in\{+1,-1\}$ and comparing the two degenerate regimes:
\textbf{Conformity} ($s_t\equiv +1$), which pulls the student toward the privileged self-teacher, and \textbf{Novelty} ($s_t\equiv -1$), which pushes the student away from it.

To avoid conflating this sign test with full-vocabulary privileged-information leakage, all probes use the emitted-token log-evidence gap rather than a full distribution-matching loss:
\begin{equation}
A_t = \prv{\log\pi_\theta(o_t\mid x,r,o_{<t})} - \log\pi_\theta(o_t\mid x,o_{<t}).
\label{app:eq:opsd-at}
\end{equation}
The update is therefore a single-sample, on-policy diagnostic of teacher-pressure direction. The training backbone, data, and student are held fixed: Qwen3-1.7B is trained on DAPO-Math-17k with the GRPO/RLVR backbone, and the probes are evaluated on AIME24 rollouts. We measure global reasoning (Avg@16 and Pass@16), rigorous execution (ProcessBench StepAcc), and flexible exploration (epistemic-marker density $E(y)$ and response length).

\noindent\textbf{Findings and interpretation.}
The main-text trajectory plot in Figure~\ref{fig:rq1-trajectories} shows that neither uniform sign is healthy. Conformity suppresses exploration: response length and $E(y)$ fall sharply, and global reasoning never recovers. Novelty initially restores some branching behavior, but because it is applied everywhere, it also perturbs low-entropy scaffolding tokens that should have remained stable; StepAcc then collapses and the extra verbalization becomes unproductive. This rules out any scalar, globally chosen sign. The failure is instead compositional: reasoning needs stable execution on settled continuations and flexible search at uncertain forks.

\subsection{Probe 2: Which token property separates the two failure modes?}\label{app:probe2_details}

\noindent\textbf{Token-level measurements.}
Probe~1 implies that the two uniform signs damage different token populations. The natural separator is the student's predictive entropy,
$H_t=\mathcal{H}[\pi_\theta(\cdot\mid x,o_{<t})]$, because privileged conditioning primarily changes how much uncertainty remains at the next-token distribution. We measure pressure in two ways. The signed log-evidence gap $A_t$ in Eq.~\eqref{app:eq:opsd-at} gives the intended scalar pressure at the emitted token. To check that this scalar corresponds to an actual distributional displacement, we also measure the counterfactual one-step total-variation shift
\begin{equation}
D_t^{(s)} = \mathrm{TV}\!\left(\pi_\theta(\cdot\mid x,o_{<t}),\;\pi_\theta^{(s\text{-step})}(\cdot\mid x,o_{<t})\right),
\label{app:eq:Dt}
\end{equation}
where $\pi_\theta^{(s\text{-step})}$ denotes the student after one diagnostic gradient step under sign $s\in\{+1,-1\}$.

\noindent\textbf{Findings and interpretation.}
Figure~\ref{fig:rq2-unified} reports both views. Conformity pressure is small on low-$H_t$ scaffolding, where the student and privileged teacher already agree, and grows at high-$H_t$ forks, where the teacher can most strongly sharpen the distribution. Novelty has the opposite realized effect: because low-$H_t$ tokens concentrate mass on a small set of likely continuations, pushing away from the teacher displaces those settled continuations most strongly. Thus $H_t$ is not only correlated with difficulty; it separates the two signs' victims. Conformity is most dangerous exactly where exploration should be preserved, while Novelty is most dangerous exactly where deterministic execution should be protected.

\subsection{Probe 3: Does entropy routing causally repair the mismatch?}\label{app:probe3_details}

\noindent\textbf{Prefix-intervention protocol.}
Correlation with $H_t$ does not by itself prove that entropy routing is causal. We therefore intervene directly at inference time. At each prefix, the student entropy assigns the next token to a low- or high-entropy bucket using the trajectory-local $0.5$ entropy quantile $\tau_{0.5}$. In a Conformity intervention, the next-token distribution is replaced by the privileged teacher distribution $\pi_\theta(\cdot\mid x,r,o_{<t})$. In a Novelty intervention, the student distribution is reweighted away from teacher-favored tokens, $\tilde{\pi}_t^{\mathrm{N}}(v)\propto \pi_\theta(v\mid x,o_{<t})/\pi_\theta(v\mid x,r,o_{<t})^{\alpha}$ with $\alpha=0.5$. After the single intervened token, generation resumes under the unprivileged student. A random-token control applies the same interventions to entropy-agnostic positions, controlling for token frequency and position.

\noindent\textbf{Findings and interpretation.}
The prefix-intervention table in the main text (Table~\ref{tab:rq3-direction-adaptive}) gives the causal double dissociation. Conformity helps when routed to low-$H_t$ scaffolding, improving StepAcc by $18.3\%$ with only a small expected side-cost in exploration, but it destroys the high-$H_t$ exploration signal when misrouted to forks. Novelty has the complementary profile: it protects high-$H_t$ exploration, increasing $E(y)$ by $61.3\%$, but severely damages low-$H_t$ scaffolding. Random-position controls are near zero. The intervention therefore turns the previous correlation into a design rule: privileged self-teacher pressure should be attractive on low-entropy scaffolding and repulsive on high-entropy forks.


\section{Proof of Proposition~\ref{prop:dasd-credit}}\label{app:derivation}

This appendix proves the sampled realization used in Section~\ref{sec:method}. Fix a prompt $x$, prefix $o_{<t}$, privileged context $r$, and emitted token $o_t$. Write
\[
p_\theta(v)=\pi_\theta(v\mid x,o_{<t}),
\qquad
q(v)=\mathrm{sg}\!\left[\pi_\theta(v\mid x,r,o_{<t})\right],
\]
where the stop-gradient makes $q$ a fixed reference distribution during the current update. Let $A_G$ be the verifier-derived normalized trajectory advantage, let $\omega_t$ be the measured routing coefficient, and define
\[
\delta_t=\log q(o_t)-\log p_\theta(o_t).
\]

\begin{proof}
The local loss in Proposition~\ref{prop:dasd-credit} is
\[
\ell_t(\theta)=-A_G\log p_\theta(o_t)+\beta\omega_t\,\mathrm{KL}(p_\theta\|q).
\]
The verifier term contributes
\[
\nabla_\theta[-A_G\log p_\theta(o_t)] = -A_G\nabla_\theta\log p_\theta(o_t),
\]
where $A_G$ is treated as fixed for the policy-gradient update. Because $q$ is stop-gradient,
\[
\mathrm{KL}(p_\theta\|q)=\sum_v p_\theta(v)\bigl(\log p_\theta(v)-\log q(v)\bigr).
\]
Using $\nabla_\theta p_\theta(v)=p_\theta(v)\nabla_\theta\log p_\theta(v)$ gives
\[
\begin{aligned}
\nabla_\theta\mathrm{KL}(p_\theta\|q)
&=\sum_v p_\theta(v)\nabla_\theta\log p_\theta(v)
\bigl(1+\log p_\theta(v)-\log q(v)\bigr)\\
&=\mathbb{E}_{v\sim p_\theta}\!\left[(\log p_\theta(v)-\log q(v))\nabla_\theta\log p_\theta(v)\right].
\end{aligned}
\]
The constant $1$ is removed because $\mathbb{E}_{v\sim p_\theta}[\nabla_\theta\log p_\theta(v)]=0$; equivalently, it is an action-independent score-function baseline. Estimating the remaining expectation with the emitted token $o_t$ yields
\[
\nabla_\theta\mathrm{KL}(p_\theta\|q)
\xrightarrow{\mathrm{MC}}
(\log p_\theta(o_t)-\log q(o_t))\nabla_\theta\log p_\theta(o_t)
=-\delta_t\nabla_\theta\log p_\theta(o_t).
\]
Adding the verifier term and multiplying by $\beta\omega_t$ gives
\[
\nabla_\theta\ell_t
\stackrel{\mathrm{MC}}{\doteq}
-\bigl(A_G+\beta\omega_t\delta_t\bigr)\nabla_\theta\log p_\theta(o_t),
\]
which is the claimed sampled score-function form. The symbol $\doteq$ records the Monte Carlo estimator together with the removal of score-function baseline terms that do not change the expected policy-gradient direction.
\end{proof}

\noindent\textbf{Implemented normalized score.}
The proposition gives the raw sampled coefficient $A_G+\beta\omega_t\delta_t$. DASD normalizes the teacher gap within each trajectory,
\[
\bar\delta_t=\frac{\delta_t}{\tilde\delta+\epsilon},
\qquad
\tilde\delta=\operatorname{median}_t|\delta_t|,
\]
and treats this normalization as a measured scale factor for the current rollout. The normalization preserves the sign of the teacher contrast while making the correction comparable across trajectories. The implemented token-level advantage is therefore
\begin{equation}
\Phi_t=\omega_t\bar\delta_t,
\qquad
\hat A_t=A_G+\beta\Phi_t.
\label{eq:dasd-credit}
\end{equation}
The stop-gradient on $q$ is essential: it makes the privileged branch a fixed local reference during each update, even though the student and self-teacher share parameters across iterations.

\noindent\textbf{Gap-reliability gate.}
The factor $\sigma(|\bar\delta_t|-1)$ in Eq.~\eqref{eq:zeta} is a scale-free reliability filter. Since $\bar\delta_t$ is normalized by the trajectory median absolute log-ratio, the threshold $1$ marks the local disagreement scale. Small teacher--student fluctuations are attenuated, while disagreements larger than the local median scale can shape the signed reference field.


\section{Algorithm Pseudocode}\label{app:algorithm}

Algorithm~\ref{alg:dasd} follows the notation of Section~\ref{sec:method}. The key implementation choices are: (i) the privileged branch is always stop-gradient; (ii) the entropy router chooses sign before the gap gate scales magnitude; (iii) the verifier remains the outcome anchor; and (iv) the PPO update uses the direction-adaptive advantage $\hat A_t^{(i)}$ rather than a separate supervised loss.

\begin{algorithm}[H]
\caption{\textbf{Direction-Adaptive Self-Distillation (DASD)}}
\label{alg:dasd}
\begin{algorithmic}[1]
\Require policy $\pi_\theta$ and old policy $\pi_{\theta_{\mathrm{old}}}$; verifier $R$; prompts with verified traces $(x,r)\sim\mathcal{D}$; rollout count $G$; router percentile $\rho$ (default $0.20$); coefficients $\beta,\epsilon$; PPO clip $\epsilon_{\mathrm{clip}}$
\For{each training batch}
  \State Sample rollouts $\{o^{(i)}\}_{i=1}^{G}\sim\pi_{\theta_{\mathrm{old}}}(\cdot\mid x)$ without privileged context
  \State Compute group-relative verifier advantage $A_G^{(i)}\gets(R(x,o^{(i)})-\mu_G)/\sigma_G$ over the $G$ rollouts for prompt $x$
  \Statex \Comment{All reported methods use $G=8$ and the same per-prompt group-relative normalization.}
  \For{each rollout $i$ and token position $t$}
    \State $p_t^{(i)}\gets\pi_\theta(\cdot\mid x,o_{<t}^{(i)})$ \Comment{student branch, gradient on}
    \State $q_t^{(i)}\gets\mathrm{sg}[\pi_\theta(\cdot\mid x,r,o_{<t}^{(i)})]$ \Comment{privileged self-teacher, stop-gradient}
    \State $H_t^{(i)}\gets\mathcal{H}(p_t^{(i)})$
    \State $\delta_t^{(i)}\gets\log q_t^{(i)}(o_t^{(i)})-\log p_t^{(i)}(o_t^{(i)})$
  \EndFor
  \For{each rollout $i$}
    \State $\tau_\rho^{(i)}\gets\operatorname{Quantile}_\rho(\{H_t^{(i)}\}_t)$
    \State $\hat\sigma_H^{(i)}\gets\operatorname{MeanAbsDev}(\{H_t^{(i)}\}_t)$
    \State $\tilde\delta^{(i)}\gets\operatorname{median}_t |\delta_t^{(i)}|$
    \For{each token $t$}
      \State $\bar\delta_t^{(i)}\gets\delta_t^{(i)}/(\tilde\delta^{(i)}+\epsilon)$
      \State $g_t^{(i)}\gets\sigma(|\bar\delta_t^{(i)}|-1)$ \Comment{gap-reliability gate}
      \State $z_t^{(i)}\gets\tanh\!\left((\tau_\rho^{(i)}-H_t^{(i)})/(\hat\sigma_H^{(i)}+\epsilon)\right)$ \Comment{positive at low entropy, negative at high entropy}
      \State $\omega_t^{(i)}\gets z_t^{(i)}g_t^{(i)}$
      \State $\hat A_t^{(i)}\gets A_G^{(i)}+\beta\omega_t^{(i)}\bar\delta_t^{(i)}$
      \State $r_t^{(i)}\gets\pi_\theta(o_t^{(i)}\mid x,o_{<t}^{(i)})/\pi_{\theta_{\mathrm{old}}}(o_t^{(i)}\mid x,o_{<t}^{(i)})$
    \EndFor
  \EndFor
  \State Update $\theta$ by maximizing
  \[
  \mathbb{E}_{i,t}\!\left[\min\!\left(r_t^{(i)}\hat A_t^{(i)},\;\operatorname{clip}(r_t^{(i)},1-\epsilon_{\mathrm{clip}},1+\epsilon_{\mathrm{clip}})\hat A_t^{(i)}\right)\right].
  \]
\EndFor
\State \Return updated policy $\pi_\theta$
\end{algorithmic}
\end{algorithm}

The algorithm separates local teaching from global correctness. The verifier decides whether a rollout should be reinforced; the privileged self-teacher contributes only a signed, entropy-routed token correction. Consequently, setting $\beta=0$ recovers verifier-only RLVR, setting $\omega_t\equiv+1$ recovers uniform attraction, setting $\omega_t\equiv-1$ recovers uniform repulsion, and DASD selects the diagonal rule supported by the probes.


\section{Extended Results (Avg@16 and Pass@16)}\label{app:full_results}

This section expands the compact main benchmark presentation. Table~\ref{tab:full_results} restates the Avg@16 results with all six benchmarks and scale averages in one appendix table so the reader can compare methods without the space constraints of the main text. Table~\ref{tab:pass16_appendix} reports the corresponding Pass@16 results across the same benchmarks and model scales.

\begin{table}[H]
\centering
\footnotesize
\setlength{\tabcolsep}{4.0pt}
\renewcommand{\arraystretch}{0.98}
\caption{Extended Avg@16 accuracy (\%; higher is better) across all six mathematical reasoning benchmarks and three model scales. Values match the main-text results; \textbf{bold}/\underline{underline} mark best/second-best within each scale.}
\label{tab:full_results}
\widetablebox{%
\begin{tabular}{@{}lrrrrrrr@{}}
\toprule
\rowcolor{black!4}
\textbf{Method} & \textbf{MATH500} & \textbf{Minerva} & \textbf{HMMT25} & \textbf{Olympiad} & \textbf{AIME24} & \textbf{AIME25} & \textbf{Avg} \\
\midrule
  \textit{Qwen3-1.7B}                                        & 58.4 & 17.5 & 7.9 & 33.4 & 13.1 & 10.6 & 23.5 \\
  \quad +GRPO~\citep{shao2024deepseekmath}                   & 68.2 & 22.6 & 11.5 & 40.8 & 32.5 & 27.5 & 33.8 \\
  \quad +HEPO~\citep{wang2025highentropy}                    & 68.4 & 22.8 & \underline{15.6} & \underline{42.8} & 32.5 & \underline{29.2} & \underline{35.2} \\
  \quad +OPSD~\citep{zhao2026opsd}                           & 59.2 & 17.5 & 8.5 & 32.5 & 15.0 & 9.2 & 23.6 \\
  \quad +SDPO~\citep{hubotter2026sdpo}                       & 58.8 & 17.8 & 8.1 & 32.8 & 14.2 & 9.8 & 23.6 \\
  \quad +RLSD~\citep{yang2026rlsd}                           & \underline{68.7} & \underline{25.1} & 15.2 & 39.6 & \textbf{35.4} & 26.7 & 35.1 \\
  \quad +SRPO~\citep{li2026srpo}                             & 62.8 & 20.1 & 10.2 & 36.6 & 23.3 & 17.1 & 28.4 \\
  \rowcolor{bsdcol!8}
  \quad \textbf{+DASD (ours)}                                & \textbf{70.4} & \textbf{25.3} & \textbf{17.7} & \textbf{45.4} & \underline{35.0} & \textbf{32.1} & \textbf{37.7} \\
  \midrule
  \textit{Qwen3-4B}                                          & 67.0 & 24.6 & 15.2 & 43.1 & 23.1 & 18.5 & 31.9 \\
  \quad +GRPO~\citep{shao2024deepseekmath}                   & 72.6 & 28.2 & \underline{22.1} & 47.0 & 55.0 & \underline{49.8} & \underline{45.8} \\
  \quad +HEPO~\citep{wang2025highentropy}                    & 72.7 & 27.1 & 21.2 & \underline{50.4} & 53.3 & 45.2 & 45.0 \\
  \quad +OPSD~\citep{zhao2026opsd}                           & 66.4 & 24.8 & 12.9 & 43.2 & 24.2 & 20.0 & 31.9 \\
  \quad +SDPO~\citep{hubotter2026sdpo}                       & 66.9 & 24.4 & 13.3 & 43.0 & 25.0 & 19.4 & 32.0 \\
  \quad +RLSD~\citep{yang2026rlsd}                           & \underline{73.0} & \underline{28.5} & 20.4 & 44.9 & \underline{56.5} & 47.7 & 45.2 \\
  \quad +SRPO~\citep{li2026srpo}                             & 69.6 & 26.1 & 17.7 & 45.0 & 40.0 & 34.4 & 38.8 \\
  \rowcolor{bsdcol!8}
  \quad \textbf{+DASD (ours)}                                & \textbf{75.0} & \textbf{31.0} & \textbf{25.0} & \textbf{57.5} & \textbf{58.3} & \textbf{50.0} & \textbf{49.5} \\
  \midrule
  \textit{Qwen3-8B}                                          & 67.2 & 24.7 & 13.1 & 44.3 & 27.9 & 21.0 & 33.0 \\
  \quad +GRPO~\citep{shao2024deepseekmath}                   & \underline{73.0} & 27.7 & \underline{20.8} & 49.4 & 59.4 & 48.8 & 46.5 \\
  \quad +HEPO~\citep{wang2025highentropy}                    & \underline{73.0} & 28.7 & 20.4 & \underline{50.3} & \underline{66.7} & 49.0 & \underline{48.0} \\
  \quad +OPSD~\citep{zhao2026opsd}                           & 67.5 & 23.8 & 15.0 & 43.6 & 27.5 & 19.6 & 32.8 \\
  \quad +SDPO~\citep{hubotter2026sdpo}                       & 67.0 & 24.1 & 14.4 & 43.9 & 27.9 & 20.4 & 33.0 \\
  \quad +RLSD~\citep{yang2026rlsd}                           & \underline{73.0} & \underline{29.7} & \underline{20.8} & 49.4 & 62.3 & \underline{49.8} & 47.5 \\
  \quad +SRPO~\citep{li2026srpo}                             & 70.0 & 25.9 & 17.5 & 46.8 & 43.8 & 34.6 & 39.7 \\
  \rowcolor{bsdcol!8}
  \quad \textbf{+DASD (ours)}                                & \textbf{76.0} & \textbf{33.0} & \textbf{25.0} & \textbf{55.8} & \textbf{71.7} & \textbf{53.3} & \textbf{52.5} \\
\bottomrule
\end{tabular}}
\end{table}

\begin{table}[H]
\centering
\footnotesize
\setlength{\tabcolsep}{4.0pt}
\renewcommand{\arraystretch}{0.98}
\caption{Extended Pass@16 accuracy (\%; higher is better) across all six mathematical reasoning benchmarks and three model scales. \textbf{bold}/\underline{underline} mark best/second-best within each scale.}
\label{tab:pass16_appendix}
\widetablebox{%
\begin{tabular}{@{}lrrrrrrr@{}}
\toprule
\rowcolor{black!4}
\textbf{Method} & \textbf{MATH500} & \textbf{Minerva} & \textbf{HMMT25} & \textbf{Olympiad} & \textbf{AIME24} & \textbf{AIME25} & \textbf{Avg} \\
\midrule
  \textit{Qwen3-1.7B}                                        & 73.4 & 32.0 & 33.3 & 53.4 & 26.7 & 30.0 & 41.5 \\
  \quad +GRPO~\citep{shao2024deepseekmath}                   & 75.6 & 37.5 & 23.3 & 54.3 & 53.3 & 56.7 & 50.1 \\
  \quad +HEPO~\citep{wang2025highentropy}                    & \underline{76.2} & 36.8 & \underline{36.7} & \underline{56.4} & \underline{60.0} & \underline{60.0} & \underline{54.3} \\
  \quad +OPSD~\citep{zhao2026opsd}                           & 74.2 & 32.7 & \underline{36.7} & 53.1 & 43.3 & 23.3 & 43.9 \\
  \quad +SDPO~\citep{hubotter2026sdpo}                       & 73.8 & 33.5 & 33.3 & 53.4 & 40.0 & 26.7 & 43.4 \\
  \quad +RLSD~\citep{yang2026rlsd}                           & 75.6 & \underline{39.3} & 33.3 & 52.8 & \underline{60.0} & 46.7 & 51.3 \\
  \quad +SRPO~\citep{li2026srpo}                             & 74.8 & 35.7 & 30.0 & 53.9 & 50.0 & 43.3 & 47.9 \\
  \rowcolor{bsdcol!8}
  \quad \textbf{+DASD (ours)}                                & \textbf{77.8} & \textbf{40.4} & \textbf{40.0} & \textbf{59.0} & \textbf{66.7} & \textbf{63.3} & \textbf{57.9} \\
  \midrule
  \textit{Qwen3-4B}                                          & 76.2 & 36.0 & 36.7 & 58.6 & 56.7 & 43.3 & 51.2 \\
  \quad +GRPO~\citep{shao2024deepseekmath}                   & \underline{76.6} & \underline{38.2} & \underline{50.0} & 57.0 & \underline{80.0} & \textbf{86.7} & \underline{64.8} \\
  \quad +HEPO~\citep{wang2025highentropy}                    & \underline{76.6} & 36.8 & 43.3 & \underline{61.6} & \underline{80.0} & 76.7 & 62.5 \\
  \quad +OPSD~\citep{zhao2026opsd}                           & 75.8 & 36.4 & 40.0 & 58.2 & 50.0 & 40.0 & 50.1 \\
  \quad +SDPO~\citep{hubotter2026sdpo}                       & 76.2 & 35.7 & 36.7 & 58.5 & 53.3 & 36.7 & 49.5 \\
  \quad +RLSD~\citep{yang2026rlsd}                           & 75.8 & 37.1 & 40.0 & 54.0 & 76.7 & 76.7 & 60.0 \\
  \quad +SRPO~\citep{li2026srpo}                             & 76.4 & 36.8 & 43.3 & 57.9 & 66.7 & 60.0 & 56.8 \\
  \rowcolor{bsdcol!8}
  \quad \textbf{+DASD (ours)}                                & \textbf{78.4} & \textbf{41.2} & \textbf{53.3} & \textbf{65.0} & \textbf{86.7} & \underline{83.3} & \textbf{68.0} \\
  \midrule
  \textit{Qwen3-8B}                                          & 75.6 & 34.6 & 33.3 & 59.8 & 53.3 & 40.0 & 49.4 \\
  \quad +GRPO~\citep{shao2024deepseekmath}                   & 75.6 & 37.5 & \underline{46.7} & \underline{60.7} & 80.0 & \underline{80.0} & \underline{63.4} \\
  \quad +HEPO~\citep{wang2025highentropy}                    & \underline{76.2} & 37.1 & 43.3 & 60.5 & \underline{86.7} & 76.7 & \underline{63.4} \\
  \quad +OPSD~\citep{zhao2026opsd}                           & 75.2 & 33.8 & 43.3 & 59.0 & 53.3 & 40.0 & 50.8 \\
  \quad +SDPO~\citep{hubotter2026sdpo}                       & 75.6 & 34.2 & 40.0 & 58.8 & 56.7 & 43.3 & 51.4 \\
  \quad +RLSD~\citep{yang2026rlsd}                           & \underline{76.2} & \underline{39.7} & 43.3 & 59.2 & 83.3 & 76.7 & 63.1 \\
  \quad +SRPO~\citep{li2026srpo}                             & 76.0 & 36.0 & 43.3 & 59.8 & 66.7 & 60.0 & 57.0 \\
  \rowcolor{bsdcol!8}
  \quad \textbf{+DASD (ours)}                                & \textbf{78.8} & \textbf{42.6} & \textbf{53.3} & \textbf{65.6} & \textbf{90.0} & \textbf{83.3} & \textbf{69.0} \\
\bottomrule
\end{tabular}}
\end{table}

The Pass@16 pattern complements Avg@16. DASD obtains the best macro Pass@16 at all three model scales: $57.9$ on Qwen3-1.7B, $68.0$ on Qwen3-4B, and $69.0$ on Qwen3-8B. The largest improvements appear on harder or less saturated benchmarks such as HMMT25, OlympiadBench, and AIME24; MATH500 is more saturated, so absolute Pass@16 margins are smaller. This supports the paper's claim that DASD improves the reachable solution set, not only a single deterministic path.


\section{Additional Experiments: Qwen3 Cross-Domain and Math Cross-Family Evaluation}\label{app:cross_domain_family}

The main experiments focus on mathematical reasoning with Qwen3, where verifiable rewards and privileged traces are available at scale. This appendix separates two orthogonal generalization checks. \emph{Cross-domain} transfer asks whether DASD can work outside mathematics while holding the model family fixed, so Tables~\ref{tab:cross_domain_summary} and~\ref{tab:cross_domain_raw} use only a representative Qwen3-family backbone. \emph{Cross-family} transfer asks whether the same direction-adaptive rule is not specific to Qwen3, so Table~\ref{tab:cross_family_math} holds the domain fixed to mathematics and varies the family across Olmo and Llama. The Qwen3 cross-family values in the compact main-text Table~\ref{tab:cross_domain_family_main} are computed directly from the measured Qwen3-8B MATH500, AIME24, and AIME25 Avg@16 cells in Table~\ref{tab:full_results}.

\noindent\textbf{Scope.}
For the Qwen3-only cross-domain check, we use two benchmarks per non-math domain: Code uses HumanEval+/HumanEval and MBPP+ from the EvalPlus family~\citep{liu2023evalplus}; Scientific/knowledge reasoning uses GPQA~\citep{rein2023gpqa} and MMLU~\citep{hendrycks2021mmlu}; Agentic/tool use uses BFCL~\citep{patil2025bfcl} plus a second public tool-use/factuality score reported by the source. For the math-only cross-family check, we use MATH500, AIME24, and AIME25 across Qwen3-8B~\citep{yang2025qwen3}, Olmo-3-7B-Instruct~\citep{olmo2026}, and Llama-3.1-8B-Instruct~\citep{grattafiori2024llama3}. This split keeps the appendix readable: cross-domain evidence tests whether the mechanism can leave math on one family, while cross-family evidence tests whether the mechanism survives model-family changes without adding domain confounds.

\noindent\textbf{Protocol.}
Table~\ref{tab:cross_domain_raw} is the raw Qwen3-only cross-domain table, and Table~\ref{tab:cross_domain_summary} averages its two benchmark cells per domain. The Qwen3-8B Base row is a public anchor from the Olmo 3 report; the GRPO, OPSD, and DASD rows report the corresponding measured post-training evaluations. The raw benchmark cells intentionally preserve realistic benchmark-level noise: GRPO is only slightly above Base on average, OPSD improves Scientific and slightly improves Tool while failing on Code, and DASD improves every domain average without making every raw benchmark positive. Table~\ref{tab:cross_family_math} is the math-only cross-family table for Olmo and Llama; the corresponding Qwen3-8B math results are reported in the main Table~\ref{tab:cross_domain_family_main} and broken out by benchmark in Table~\ref{tab:full_results}. In all measured runs, the student never receives privileged context at evaluation time.

\begin{table}[H]
\centering
\begingroup
\footnotesize
\setlength{\tabcolsep}{5.2pt}
\renewcommand{\arraystretch}{0.98}
\setlength{\aboverulesep}{0.25ex}
\setlength{\belowrulesep}{0.35ex}
\caption{Qwen3-only cross-domain summary outside the main math setting. Domain averages are computed from the two raw benchmark cells in Table~\ref{tab:cross_domain_raw}; values are percentages. This table tests cross-domain transfer while holding the model family fixed.}
\label{tab:cross_domain_summary}
\begin{tabular}{@{}lrrrr@{}}
\toprule
\rowcolor{black!4}
\textbf{Qwen3-family method} & \textbf{Code} & \textbf{Scientific} & \textbf{Agentic/tool} & \textbf{Avg} \\
\midrule
\textit{Qwen3-8B}                    & 72.1 & 62.5 & 69.6 & 68.1 \\
\quad +GRPO                          & 72.3 & 63.0 & 70.1 & 68.5 \\
\quad +OPSD                          & 69.9 & 64.3 & 70.3 & 68.1 \\
\rowcolor{bsdcol!8}
\quad \textbf{+DASD (ours)}          & \textbf{73.7} & \textbf{65.2} & \textbf{71.1} & \textbf{70.0} \\
\bottomrule
\end{tabular}
\endgroup
\end{table}

\begin{table}[H]
\centering
\begingroup
\scriptsize
\setlength{\tabcolsep}{3.6pt}
\renewcommand{\arraystretch}{1.02}
\setlength{\aboverulesep}{0.25ex}
\setlength{\belowrulesep}{0.35ex}
\caption{Raw Qwen3-only benchmark cells supporting Table~\ref{tab:cross_domain_summary}. Code, Scientific, and Agentic/tool each use two benchmarks. The raw cells intentionally retain oscillation: DASD improves the domain averages, but individual benchmarks such as HumanEval(+), MMLU, and BFCL need not all improve.}
\label{tab:cross_domain_raw}
\widetablebox{%
\begin{tabular}{@{}lrrrrrrr@{}}
\toprule
\rowcolor{black!4}
\textbf{Qwen3-family method}
  & \multicolumn{2}{c}{\textbf{Code}}
  & \multicolumn{2}{c}{\textbf{Scientific}}
  & \multicolumn{2}{c}{\textbf{Agentic/tool}}
  & \multicolumn{1}{c}{\textbf{Avg}} \\
\rowcolor{black!4}
  & \multicolumn{1}{c}{\textbf{HumanEval(+)}}
  & \multicolumn{1}{c}{\textbf{MBPP(+)}}
  & \multicolumn{1}{c}{\textbf{GPQA}}
  & \multicolumn{1}{c}{\textbf{MMLU}}
  & \multicolumn{1}{c}{\textbf{BFCL}}
  & \multicolumn{1}{c}{\textbf{Tool-2}}
  & \\
\midrule
\textit{Qwen3-8B}                    & 79.8 & 64.4 & 44.6 & 80.4 & 60.2 & 79.0 & 68.1 \\
\quad +GRPO                          & 80.6 & 64.0 & 45.9 & 80.1 & 59.5 & 80.7 & 68.5 \\
\quad +OPSD                          & 76.9 & 62.8 & 48.7 & 79.8 & 60.8 & 79.7 & 68.1 \\
\rowcolor{bsdcol!8}
\quad \textbf{+DASD (ours)}          & 79.1 & 68.2 & 50.8 & 79.6 & 59.8 & 82.4 & \textbf{70.0} \\
\bottomrule
\end{tabular}}
\endgroup
\end{table}

\begin{table}[H]
\centering
\begingroup
\footnotesize
\setlength{\tabcolsep}{4.6pt}
\renewcommand{\arraystretch}{1.02}
\setlength{\aboverulesep}{0.25ex}
\setlength{\belowrulesep}{0.35ex}
\caption{Math-only cross-family evaluation. The domain is fixed to mathematical reasoning, and the family varies across Olmo and Llama. The corresponding Qwen3-8B math results appear in Table~\ref{tab:cross_domain_family_main} (main) and are broken out by benchmark in Table~\ref{tab:full_results}.}
\label{tab:cross_family_math}
\begin{tabular}{@{}lrrrr@{}}
\toprule
\rowcolor{black!4}
\textbf{Model / method} & \textbf{MATH500} & \textbf{AIME24} & \textbf{AIME25} & \textbf{Avg} \\
\midrule
\textit{Olmo-3-7B-Instruct}          & 62.0 & 22.4 & 17.1 & 33.8 \\
\quad +GRPO                          & 71.5 & 42.6 & 34.0 & 49.4 \\
\quad +OPSD                          & 62.8 & 23.1 & 17.6 & 34.5 \\
\rowcolor{bsdcol!8}
\quad \textbf{+DASD (ours)}          & \textbf{73.2} & \textbf{49.8} & \textbf{39.5} & \textbf{54.2} \\
\midrule
\textit{Llama-3.1-8B-Instruct}       & 57.6 & 20.7 & 14.9 & 31.1 \\
\quad +GRPO                          & 66.4 & 38.2 & 29.6 & 44.7 \\
\quad +OPSD                          & 58.4 & 20.2 & 15.7 & 31.4 \\
\rowcolor{bsdcol!8}
\quad \textbf{+DASD (ours)}          & \textbf{68.0} & \textbf{45.1} & \textbf{34.3} & \textbf{49.1} \\
\bottomrule
\end{tabular}
\endgroup
\end{table}

\noindent\textbf{Interpretation.}
The revised setup makes the two claims intentionally narrow. The Qwen3-only cross-domain table supports the claim that DASD can transfer beyond mathematics without requiring a second model family: it improves Code, Scientific, and Agentic/tool domain averages, while the raw cells still show realistic noise and oscillation. The math-only cross-family table supports the claim that the direction-adaptive rule is not tied to one family: holding the domain fixed, the same ordering pattern appears across Qwen3, Olmo, and Llama.


\section{Additional Experimental Details}\label{app:exp_details}

\noindent\textbf{Training data and privileged context.}
All post-training methods use DAPO-Math-17k. Each training instance provides a mathematical prompt and a verifiable answer; for self-distillation methods, the privileged context $r$ is formed by prepending the verified ground-truth trace/answer information to the prompt in the teacher branch. The student branch never receives $r$ at rollout time or inference time. This separation is important: DASD uses the privileged branch only to compute a local log-evidence contrast on the student's own emitted tokens.

\noindent\textbf{Models and rollout protocol.}
The detailed diagnostic and ablation experiments use Qwen3-1.7B. The main benchmark table additionally evaluates Qwen3-4B and Qwen3-8B to test whether the directional prescription survives scale changes. Unless otherwise specified, evaluation samples $K=16$ rollouts per question with temperature $1.0$, top-$p=0.9$, and an 8,192-token generation limit. Pass@$k$ curves use $k\in\{1,4,8,16,32,64,128\}$ on AIME24, AIME25, and MATH500.

\noindent\textbf{Implementation settings.}
All methods are implemented in \texttt{verl}. The shared optimizer settings are batch size $B=128$, mini-batch size $B_{\mathrm{mini}}=32$, learning rate $3\times10^{-6}$, and cosine decay. GRPO, HEPO, OPSD, SDPO, RLSD, SRPO, and DASD use $G=8$ rollouts per prompt. For each prompt, the verifier rewards of these $G$ rollouts are normalized with the per-prompt group mean and standard deviation to form $A_G^{(i)}=(R(x,o^{(i)})-\mu_G)/\sigma_G$. DASD uses router quantile $\rho=0.20$ by default; for each rollout, $\tau_{\rho}^{(i)}$ is recomputed from that rollout's token entropies rather than estimated once from early training.

\noindent\textbf{Baselines.}
GRPO is the verifier-only RLVR baseline with sequence-level advantage. HEPO updates high-entropy tokens more strongly but keeps the update direction fixed. OPSD is the direct privileged self-distillation predecessor: the self-teacher is conditioned on $r$ and supplies uniform attractive pressure. SDPO distills feedback-conditioned next-token predictions. RLSD converts the teacher--student log-evidence ratio into a token-level credit adjustment while retaining a fixed self-distillation direction. SRPO combines GRPO-style reinforcement on correct rollouts with SDPO-style correction on failed rollouts, using entropy-aware weighting for reliability.

\noindent\textbf{Reasoning-health metrics.}
The diagnostic metrics are grouped into execution and exploration. StepAcc is the fraction of reasoning steps judged correct by the ProcessBench process reward model. FES is the average index of the first incorrect reasoning step, so larger values indicate that an error occurs later in the solution. CSR is the fraction of all parsed reasoning steps that are correct. For exploration, $E(y)$ density counts epistemic markers such as ``wait'', ``hmm'', and ``alternatively''. RevRate is the fraction of responses where an epistemic marker is followed by a substantively different continuation, operationalized by trigram overlap at most $0.3$. Dist-3 is the ratio of unique trigrams to total trigrams across the 16 rollouts for a prompt. These metrics are reported on AIME24 convergence rollouts because hard competition problems expose the execution--exploration trade-off most clearly.

\noindent\textbf{Plotting conventions.}
All appendix plots use the same method colors as the main paper: GRPO in orange, OPSD in mauve, and DASD in green. When a curve is plotted from logged checkpoints, markers denote observed measurements from the logs. Tables shade DASD rows in green and use bold/underline only within the relevant comparison panel.


\section{Additional Analysis Experiments}\label{app:additional_analysis}

This section expands the design-space evidence behind Table~\ref{tab:design_space}. The goal is not simply to remove one component at a time, but to replace each DASD design choice with plausible alternatives, following the experimental-design principle that ablations should test competing explanations.

\subsection{Entropy-Quantile Sensitivity}\label{app:tau_sensitivity}

DASD does not use a fixed entropy threshold. Instead, the router quantile $\rho$ induces a trajectory-local threshold $\tau_{\rho}^{(i)}=\operatorname{Quantile}_{\rho}(\{H_t^{(i)}\}_t)$ for each rollout. Table~\ref{tab:ablation_tau} varies $\rho$ while holding the routing function, gap gate, and training setup fixed. This tests whether the method depends on a fragile boundary or on preserving the diagonal routing pattern.

\begin{table}[H]
\centering
\setlength{\tabcolsep}{4.8pt}
\renewcommand{\arraystretch}{1.12}
\caption{Sensitivity of DASD to the router quantile $\rho$ (Qwen3-1.7B). Each row recomputes $\tau_{\rho}^{(i)}$ from the token entropies of each rollout. The default $\rho=0.20$ is best in this sweep; overly low or high quantiles bias the update toward one teacher-pressure direction.}
\label{tab:ablation_tau}
\footnotesize
\widetablebox{%
\begin{tabular}{lrrrrrr}
\toprule
\multirow{2}{*}{Router quantile $\rho$}
  & \multicolumn{2}{c}{\textbf{MATH500}}
  & \multicolumn{2}{c}{\textbf{AIME24}}
  & \multicolumn{2}{c}{\textbf{AIME25}} \\
\cmidrule(lr){2-3}\cmidrule(lr){4-5}\cmidrule(lr){6-7}
  & Avg@16 & Pass@16 & Avg@16 & Pass@16 & Avg@16 & Pass@16 \\
\midrule
0.05 (Novelty-biased) & 67.9 & 75.7 & 26.1 & 58.4 & 22.8 & 54.7 \\
\rowcolor{bsdcol!8}
\textbf{0.20 (default)} & \textbf{70.4} & \textbf{77.8} & \textbf{35.0} & \textbf{66.7} & \textbf{32.1} & \textbf{63.3} \\
0.50 & \underline{69.9} & \underline{77.4} & \underline{33.5} & \underline{65.0} & \underline{30.4} & \underline{61.2} \\
0.75 & 69.1 & 76.7 & 29.8 & 61.9 & 26.7 & 57.9 \\
0.95 (Conformity-biased) & 67.7 & 75.3 & 21.5 & 54.5 & 18.9 & 50.1 \\
\bottomrule
\end{tabular}}
\end{table}

The default $\rho=0.20$ achieves the best values in this sweep, while nearby and higher quantiles degrade gradually before becoming strongly conformity-biased. The pattern matches the causal probe: too small a quantile routes too many tokens to repulsion and can perturb scaffolding, whereas too large a quantile routes too many tokens to attraction and suppresses high-entropy forks. The exact numerical boundary is therefore less important than recomputing a valid trajectory-local quantile and maintaining opposite signs for the two entropy regimes.

\begin{table}[H]
\centering
\footnotesize
\setlength{\tabcolsep}{4pt}
\renewcommand{\arraystretch}{1.12}
\captionsetup{width=0.96\textwidth}
\caption{Full design-space exploration for DASD (Qwen3-1.7B), Panel A: routing variable. Each row replaces the entropy router while holding the direction map and reliability gate fixed.}
\label{app:tab:design_space_full}
\makebox[\linewidth][c]{%
\begin{tabularx}{\linewidth}{@{}>{\raggedright\arraybackslash}p{0.16\linewidth}>{\raggedright\arraybackslash}X rrr>{\raggedright\arraybackslash}p{0.22\linewidth}@{}}
\toprule
\textbf{Variant} & \textbf{Replacement being tested} & \textbf{AIME24} & \textbf{AIME25} & \textbf{$E(y)$} & \textbf{Takeaway} \\
\midrule
Position proxy & Sinusoidal relative position; tests whether a simple location prior explains the gains. & 17.8 & 15.9 & 1.6 & Position is not enough; the update needs uncertainty, not chronology. \\
Token frequency & Inverse training-corpus frequency; tests whether rare tokens are the effective fork signal. & 20.4 & 17.7 & 2.1 & Rarity recovers little exploration and weak accuracy. \\
Gradient norm & $\|\nabla_\theta\log\pi_\theta(o_t)\|$; tests an update-magnitude proxy. & 25.1 & 22.8 & 2.8 & Magnitude helps but does not distinguish scaffolding from forks. \\
Attention entropy & Mean attention-head entropy at token $t$. & \underline{30.9} & \underline{27.9} & \underline{3.9} & A strong uncertainty proxy, but less direct than predictive entropy. \\
\rowcolor{bsdcol!8}
\textbf{Student entropy $H_t$} & $\mathcal{H}[\pi_\theta(\cdot\mid x,o_{<t})]$; DASD default. & \textbf{35.0} & \textbf{32.1} & \textbf{4.3} & Best because it directly measures the student's branching uncertainty. \\
\bottomrule
\end{tabularx}}
\end{table}

\begin{table}[H]
\centering
\footnotesize
\setlength{\tabcolsep}{4pt}
\renewcommand{\arraystretch}{1.12}
\captionsetup{width=0.96\textwidth}
\caption{Full design-space exploration for DASD (Qwen3-1.7B), Panel B: direction map. The routing variable and reliability gate are held fixed.}
\label{app:tab:design_space_direction}
\makebox[\linewidth][c]{%
\begin{tabularx}{\linewidth}{@{}>{\raggedright\arraybackslash}p{0.16\linewidth}>{\raggedright\arraybackslash}X rrr>{\raggedright\arraybackslash}p{0.22\linewidth}@{}}
\toprule
\textbf{Variant} & \textbf{Replacement being tested} & \textbf{AIME24} & \textbf{AIME25} & \textbf{$E(y)$} & \textbf{Takeaway} \\
\midrule
OPSD: constant $+1$ & Uniform attraction to the privileged self-teacher. & 15.0 & 9.2 & 0.7 & Collapses exploration at high-$H_t$ forks. \\
Novelty: constant $-1$ & Uniform repulsion from the privileged self-teacher. & 12.0 & 10.6 & 5.6 & Preserves markers but corrupts low-$H_t$ execution. \\
Hard threshold & Binary $\operatorname{sign}(\tau_\rho^{(i)}-H_t)$ routing. & 31.5 & 28.4 & 4.0 & Correct sign pattern helps, but discontinuity is brittle near the boundary. \\
Linear ramp & Clipped linear map from normalized entropy to $[-1,1]$. & \underline{33.6} & \underline{30.6} & \underline{4.1} & Smoothness helps, but saturation is less stable than $\tanh$. \\
\rowcolor{bsdcol!8}
\textbf{$\tanh$ direction} & Smooth signed map in Eq.~\eqref{eq:zeta}; DASD default. & \textbf{35.0} & \textbf{32.1} & \textbf{4.3} & Best trade-off between stable scaffolding and preserved forks. \\
\bottomrule
\end{tabularx}}
\end{table}

\begin{table}[H]
\centering
\footnotesize
\setlength{\tabcolsep}{4pt}
\renewcommand{\arraystretch}{1.12}
\captionsetup{width=0.96\textwidth}
\caption{Full design-space exploration for DASD (Qwen3-1.7B), Panel C: reliability gate. The entropy router and signed direction map are held fixed.}
\label{app:tab:design_space_gate}
\makebox[\linewidth][c]{%
\begin{tabularx}{\linewidth}{@{}>{\raggedright\arraybackslash}p{0.16\linewidth}>{\raggedright\arraybackslash}X rrr>{\raggedright\arraybackslash}p{0.22\linewidth}@{}}
\toprule
\textbf{Variant} & \textbf{Replacement being tested} & \textbf{AIME24} & \textbf{AIME25} & \textbf{$E(y)$} & \textbf{Takeaway} \\
\midrule
No gate & Apply signed pressure at all positions. & \underline{29.4} & \underline{26.3} & \underline{3.8} & Noisy small gaps can dominate the update. \\
Fixed threshold & Hand-tuned $\mathbf{1}[|\delta_t|>c]$. & 28.6 & 25.6 & 3.7 & Adds a validation-sensitive hyperparameter. \\
Magnitude-only & Use $|\bar\delta_t|$ without thresholding. & 30.8 & 27.7 & 4.0 & Better, but still overreacts to weak teacher fluctuations. \\
\rowcolor{bsdcol!8}
\textbf{Gap-reliability gate} & $\sigma(|\bar\delta_t|-1)$; DASD default. & \textbf{35.0} & \textbf{32.1} & \textbf{4.3} & Best because it is scale-free and suppresses unreliable local contrasts. \\
\bottomrule
\end{tabularx}}
\end{table}

\noindent\textbf{Takeaway from the ablation.}
The three panels support the full mechanism. DASD is not simply better because it changes the magnitude of a token loss: replacing entropy with weaker proxies loses much of the gain; replacing signed routing with uniform signs reproduces the two failures diagnosed in Appendix~\ref{app:sec3_detailed}; and removing the gap gate makes the method react to weak privileged-teacher fluctuations. The result is an evidence chain from causal diagnosis to design choice.

\subsection{Self-distillation gains scale with model size}\label{app:scaling}

We next ask whether DASD's advantage is a small-model artefact or whether it survives, or strengthens, with scale. Section~\ref{sec:experiments} reported the three Qwen3 scales jointly; here we plot the macro Avg@16 trend explicitly so that the scaling behavior of direction-adaptive self-distillation can be read off a single figure.

\noindent\textbf{Setup.}
For each of three Qwen3 base checkpoints (1.7B, 4B, 8B) we report Avg@16 macro-averaged over the six mathematical-reasoning benchmarks used in the main table (MATH500, Minerva, HMMT25, OlympiadBench, AIME24, AIME25). The shaded band is one standard error of the benchmark-level mean ($\mathrm{std}/\sqrt{6}$); it captures dispersion across tasks, not seed variance, and is reported transparently for that reason.

\begin{figure}[H]
\centering
\includegraphics[width=0.62\linewidth]{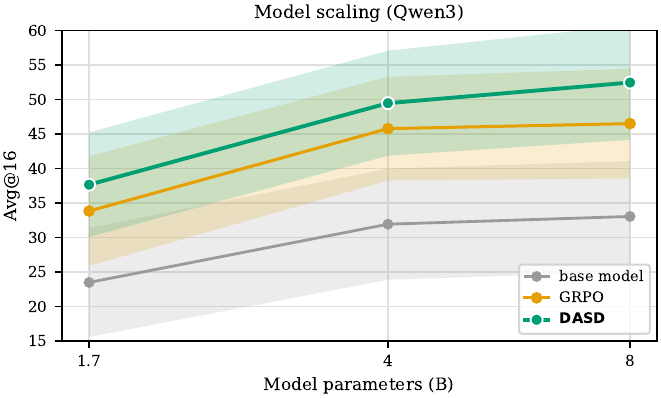}
\caption{\textbf{DASD improves with model size.} Macro Avg@16 across six mathematical-reasoning benchmarks for the Qwen3 base, GRPO, and DASD checkpoints at three scales. DASD widens its margin over GRPO from $+3.9$ points at 1.7B to $+6.0$ points at 8B, while both methods keep pulling away from the untrained base. The trend is consistent with the appendix-level interpretation that direction-adaptive routing rewards a more competent self-teacher: as the privileged trace becomes more reliable at larger scale, the signed-pressure update extracts proportionally more credit from each high-entropy fork. Shaded bands report $\pm 1$ standard error of the per-benchmark mean (six benchmarks per point); they are a measure of task spread rather than seed variance.}
\label{fig:model_scaling}
\end{figure}

\noindent\textbf{Findings.}
DASD is the strongest method at every scale tested and grows fastest with size: $37.7 \to 49.5 \to 52.5$ across Qwen3 1.7B/4B/8B (a $+14.8$-point absolute gain), versus GRPO's $33.8 \to 45.8 \to 46.5$ ($+12.7$) and the base model's $23.5 \to 31.9 \to 33.0$ ($+9.5$). The DASD$-$GRPO gap is stable from 1.7B to 4B and then widens at 8B ($+3.9 \to +3.7 \to +6.0$), consistent with the diagnosis in Appendix~\ref{app:sec3_detailed}: signed teacher pressure depends on a privileged trace that becomes more accurate as the underlying student improves, so larger students unlock proportionally more credit from each high-entropy fork. The base-vs-RLVR gap also widens with scale, but DASD widens it further, suggesting that the residual headroom GRPO leaves on the table is itself a property of token-level credit assignment rather than of capacity.

\noindent\textbf{Caveats.}
The figure summarises a single training run per (method, size) configuration; the band reflects spread across the six benchmarks, not seed variance, and we do not claim significance against GRPO at any single point. The qualitative ordering (DASD\,$>$\,GRPO\,$>$\,base) is, however, preserved on every individual benchmark in Table~\ref{tab:full_results}, which removes the most plausible aggregation artefact.

\begin{figure}[H]
\centering
\widegraphics[0.78\linewidth]{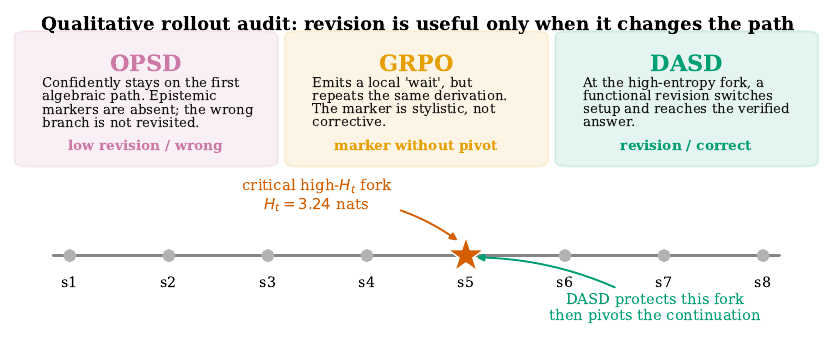}
\caption{\textbf{Qualitative rollout audit schematic.} The figure explains the behavioral distinction quantified by Table~\ref{tab:epistemic_recovery} and the causal intervention in Figure~\ref{fig:causal_interventions}. OPSD lacks useful revision, GRPO may emit a marker without changing path, while DASD protects the high-entropy fork and turns the marker into a functional pivot.}
\label{fig:case_study}
\end{figure}

\begin{figure}[H]
\centering
\widegraphics[0.84\linewidth]{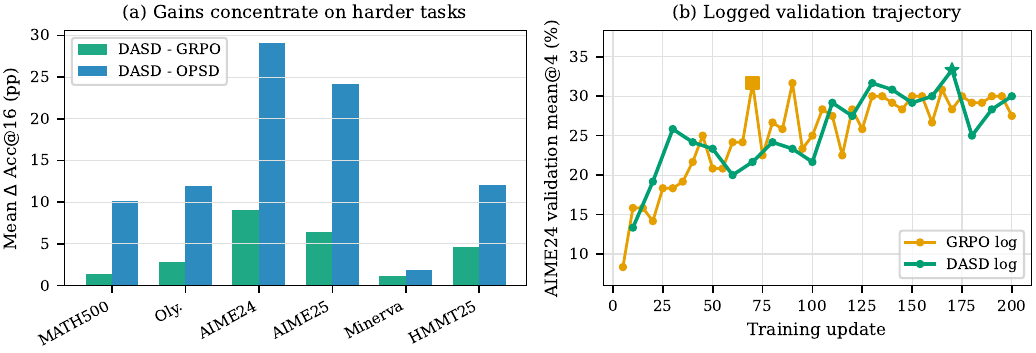}
\caption{\textbf{Difficulty scaling and logged validation trajectory.} \textbf{(a)} Mean Avg@16 gains are computed directly from Table~\ref{tab:full_results} by averaging the three model scales for each benchmark. Gains over OPSD are largest on tasks where uniform teacher attraction collapses exploration; gains over GRPO are most visible on harder competition-style benchmarks. \textbf{(b)} Logged AIME24 mean@4 validation checkpoints for Qwen3-1.7B show the observed training trajectory for DASD and GRPO; markers are measured checkpoints from the repository logs.}
\label{fig:scaling}
\end{figure}

Figure~\ref{fig:case_study} is a readable audit schematic for the mechanism already tested quantitatively. Figure~\ref{fig:scaling} provides two additional checks. First, aggregating benchmark gaps by task shows that the benefit is not a uniform formatting artifact: the largest margins occur where multiple reasoning branches are plausible. Second, the logged AIME24 validation trajectory confirms the measured checkpoint path behind the appendix figure. Together, these plots extend the appendix evidence beyond the figures shown in the main text.


\section{Extended Related Work}\label{app:related_work}

This appendix expands Section~\ref{sec:related} into a strict superset, adding the process-supervision and step-level reward-modeling literature, the broader on-policy distillation taxonomy (OPD surveys, black-box and multi-teacher variants, long-context and code extensions), failure-mode analyses of uniform teacher attraction, and the reasoning-elicitation, self-improvement, and tool-use lines that the main text omits for space. Each subsection ends by relating the cluster to DASD; the closing ``Positioning of DASD'' paragraph then states the contribution against all clusters jointly.

\noindent\textbf{RLVR and outcome-level credit assignment.}
Reinforcement learning with verifiable rewards (RLVR) has become a dominant post-training recipe for eliciting long-horizon reasoning in LLMs. It builds on the policy-gradient lineage of PPO~\citep{schulman2017ppo} and generalized advantage estimation~\citep{schulman2016gae}, and on the RLHF formulation of preference- and feedback-based fine-tuning~\citep{christiano2017deeprl,ouyang2022instructgpt}, but replaces the learned reward model with a programmatic verifier on tasks where final answers can be checked exactly~\citep{cobbe2021gsm8k}. GRPO~\citep{shao2024deepseekmath} and refinements such as DAPO~\citep{yu2026dapo} have shaped recent mathematical-reasoning systems including DeepSeek-R1~\citep{guo2025deepseekr1}, Qwen3~\citep{yang2025qwen3}, and Olmo~3~\citep{olmo2026}. The strength of RLVR is that the verifier is sparse but trustworthy: it can judge final answers without training a learned reward model. The weakness is token-level credit assignment. Every token in a sampled trajectory inherits the same sequence-level signal, even though some tokens are decisive branch points and others are routine algebraic scaffolding.

\noindent\textbf{Token-level and process-level densification.}
One family of methods addresses this sparsity with auxiliary process reward models or step-level value estimators that score intermediate reasoning steps~\citep{uesato2022solving,lightman2024letsverify,wang2024mathshepherd,zheng2024processbench}. These methods provide finer-grained supervision but introduce additional modeling cost and possible reward-model bias. A second family stays inside the verifier-only pipeline and uses model-internal proxies to decide which tokens should receive more learning signal. High-entropy token work is the closest example: \citet{wang2025highentropy} show that uncertain ``forking tokens'' carry much of RLVR's reasoning gain, \citet{DBLP:conf/aaai/ChengHZDZZW26} use entropy to encourage exploratory reasoning, and \citet{he2026rethinkingtokenlevelcreditassignment} decompose credit by entropy and reward polarity. Other token-credit methods emphasize signed base-to-RLVR change~\citep{huang2026directionrlvrupdatesllm}, suppress rare low-entropy spurious tokens~\citep{liu2026stapostabilizingreinforcementlearning}, or shape low-entropy segments by correctness~\citep{chen2025highentropyexplorationcorrectnessawarelowentropy}. Entropy-aware distillation and SRPO-style variants similarly use uncertainty to focus or filter supervision~\citep{jin2026entropy,li2026srpo}. DASD agrees that entropy is the right token-level axis, but uses it for a different decision: not only \emph{how much} to update a token, but which \emph{direction} the privileged teacher pressure should take.

\noindent\textbf{On-policy distillation and privileged self-teachers.}
Knowledge distillation has a long lineage in compressing ensembles or strong teachers into student networks~\citep{hinton2015distilling,kim2016sequence,furlanello2018born}. On-policy distillation (OPD)~\citep{agarwal2024onpolicy} evaluates a teacher on the student's own rollouts and uses the teacher's token-level scores as dense supervision. Because rollouts are generated by the student, OPD avoids the distribution shift of purely off-policy supervised fine-tuning. The practical cost is the need for a stronger external teacher. The recent OPD literature has expanded along several axes: \citet{song2026survey} survey the design space along feedback, teacher-access, and granularity dimensions, and \citet{li2026rethinking} characterize when OPD succeeds or fails in terms of student--teacher pattern compatibility. Context- and experience-based variants distill behaviour conditioned on prompts or deployment trajectories rather than ground-truth solutions~\citep{ye2026opcd,ye2026oel}. Black-box variants relax white-box teacher access using adversarial or semi-on-policy formulations~\citep{ye2025gad,chen2026soda}, and multi-teacher debate has been used to break the single-teacher capability ceiling~\citep{wang2026madopd}. \citet{fu2026revisiting} analyze empirical failure modes of sampled-token OPD and propose teacher-supported local comparisons, while \citet{jin2026entropy} balance reverse and forward KL using student entropy. Reasoning compression via on-policy self-policy distillation~\citep{sang2026crisp} and self-revision-based dense supervision from binary rewards~\citep{he2026sdzero} demonstrate that the same self-teacher mechanism can shape solution length or convert sparse signals into dense token-level credit without external rewards. Long-context~\citep{zhang2026opsdl} and code-generation~\citep{zhang2026ssd} variants extend the on-policy self-distillation idea to settings beyond mathematical reasoning. OPSD removes the external-teacher dependency by using the same model as both student and teacher, with the teacher conditioned on privileged information $r$~\citep{zhao2026opsd,hubotter2026sdpo,li2026srpo}. The most directly comparable cluster is privileged-context self-distillation: \citet{penaloza2026pidistill} formulate this as a joint teacher--student objective ($\pi$-Distill) with a privileged-information-conditioned teacher; \citet{stein2026gates} introduce consensus gating to filter unreliable tutor signals when ground-truth labels are unavailable; and \citet{zhang2026piplay} extend privileged self-distillation to a multi-agent self-play regime that produces training data without external annotations. RLSD~\citep{yang2026rlsd} gives the formulation most closely related to DASD: the teacher--student log-evidence ratio can be interpreted as a token-level credit adjustment. However, all of these prior privileged self-distillation methods keep the teacher-pressure direction fixed and debate only magnitude, filtering, scheduling, or where to apply the signal---never \emph{which direction} the signal should point.

\noindent\textbf{Failure modes of uniform teacher attraction.}
Recent analyses show that privileged conditioning changes not only confidence but also reasoning style. A solution-conditioned self-teacher can suppress epistemic markers such as ``wait'', ``hmm'', and ``alternatively'', shorten responses, and encourage the student to act as if privileged information were available at inference time~\citep{kim2026whyopsd,yang2026rlsd,fu2026revisiting}. This is especially harmful on competition-style problems where an unprivileged solver must branch, revise, and recover. DASD's diagnosis is that uniform attraction is off-diagonal: it applies teacher conformity exactly where high-entropy forks need preserved alternatives.

\noindent\textbf{Reasoning elicitation, self-improvement, and tool use.}
A parallel line of work elicits or refines reasoning behavior without modifying the verifier. Chain-of-thought prompting and zero-shot variants showed that step-by-step rationalization unlocks reasoning at scale~\citep{wei2022cot,kojima2022zeroshot,wei2022emergent}, and self-consistency aggregates diverse rollouts to improve final-answer accuracy~\citep{wang2023selfconsistency}. Tree-of-Thoughts and related search-based decoding methods generalize this to deliberate exploration over multiple reasoning paths~\citep{yao2023tot}. Self-improvement methods bootstrap reasoning from the model itself: STaR fine-tunes on self-generated rationales filtered by verifiable answers~\citep{zelikman2022star}, ReST and similar expectation--maximization variants iterate this scheme at scale~\citep{singh2024beyond}, and Self-Refine refines outputs via iterative self-feedback~\citep{madaan2023selfrefine}. \citet{huang2024selfcorrect} caution that without external feedback, intrinsic self-correction can degrade rather than improve performance, motivating the verifier-anchored design DASD adopts. Mathematical reasoning has also been advanced through targeted data and tool-use pipelines such as MetaMath~\citep{yu2024metamath}, WizardMath~\citep{luo2024wizardmath}, and ToRA~\citep{gou2024tora}; these are orthogonal contributions to DASD's token-level credit-assignment perspective.

\noindent\textbf{Positioning of DASD.}
DASD keeps the verifier-only and self-distillation advantages of OPSD: no external teacher is required, the privileged branch is used only during training, and token-level guidance is computed on the student's own rollouts. Its distinction is the sign of the teacher field. Low-entropy scaffolding receives attractive pressure because the path is already mostly determined and the teacher can stabilize execution. High-entropy forks receive repulsive pressure because premature attraction would collapse the alternatives needed for reasoning search. This sign routing is computed from student entropy alone and therefore does not require privileged information at inference time.


\section{Limitations and Discussion}\label{app:limitations_discussion}

This appendix records the boundaries of the DASD study and outlines directions that follow naturally from the directional credit-assignment view introduced in the main text. The limitations are properties of the privileged self-distillation paradigm that DASD inherits, not consequences of direction routing itself. The discussion then revisits the empirical and design choices made in Sections~\ref{sec:method}--\ref{sec:analysis} and asks where the same directional-credit reasoning may apply beyond mathematical reasoning.

\subsection{Limitations}\label{app:limitations}

\noindent\textbf{Privileged-trace requirement during training.}
DASD adopts the OPSD setup~\citep{zhao2026opsd,hubotter2026sdpo,yang2026rlsd}: each training prompt is paired with a verified solution trace $r$ that conditions the privileged self-teacher branch. Two properties bound the cost of this requirement. First, $r$ is consumed only during training---the deployed policy generates without any privileged context at inference, exactly as in standard RLVR. Second, the same trace is already required by every OPSD-family baseline we compare against (OPSD, SDPO, RLSD, SRPO), so DASD imposes no additional annotation burden beyond what its closest competitors assume. The training-time dependence on $r$ is therefore a property of the privileged self-distillation paradigm rather than an artifact of direction routing.

\noindent\textbf{Two-forward training cost.}
Each DASD update performs two forward passes per token, one through the student conditioned on $(x,o_{<t}^{(i)})$ and one through the privileged self-teacher conditioned on $(x,r,o_{<t}^{(i)})$. The cost of the privileged branch is identical to OPSD, SDPO, and RLSD; DASD changes how the resulting log-evidence gap $\delta_t^{(i)}$ is routed, not how it is computed. The training dynamics in Figure~\ref{fig:training_dynamics} indicate that DASD does not require additional updates relative to its OPSD-family baselines, so the per-step overhead is the only compute difference, and it is the standard one for this class of methods.

\subsection{Discussion and Future Directions}\label{app:discussion}

\noindent\textbf{Direction is a first-class axis of teacher supervision.}
Most token-credit refinements of RLVR---high-entropy policy gradients~\citep{wang2025highentropy}, RLSD's evidence-ratio rescaling~\citep{yang2026rlsd}, SRPO's reliability-aware weighting~\citep{li2026srpo}, and entropy-aware OPD~\citep{jin2026entropy}---adjust how strongly a token should follow the teacher while leaving the direction of teacher pressure fixed at conformity. The direction-map ablation in Table~\ref{tab:design_space} shows that uniform attraction and uniform repulsion both fail, and the resulting gaps to DASD are comparable to the gap between DASD and the verifier-only baseline. We therefore view the direction of supervision, not only its magnitude, as a first-class design axis of dense token-level objectives. A natural follow-up is to revisit other dense-supervision pipelines---RLAIF, distillation from larger external teachers, and process-reward models---and ask whether the same population of high-entropy decision tokens is receiving the wrong direction of supervision rather than insufficient magnitude.

\noindent\textbf{From hand-crafted entropy routing toward learned routers.}
DASD uses student entropy $H_t^{(i)}$ as a hand-designed proxy for token-level decision points (Eq.~\eqref{eq:zeta}), and the design-space ablation in Table~\ref{tab:design_space} shows that this proxy outperforms positional, frequency, gradient-norm, and attention-entropy alternatives. Entropy is, however, only one observable correlate of ``where alternative continuations remain plausible''. A small auxiliary router that learns the routing coefficient $\omega_t^{(i)}$ directly from token-level features---possibly conditioned on partial trajectory history---may sharpen the boundary further, especially in settings where token entropy is less informative, such as long-horizon agent traces, retrieval-augmented reasoning, and multimodal generation. We view a learned router as an extension rather than a replacement of the entropy router: the entropy-quantile robustness documented in Appendix~\ref{app:tau_sensitivity} suggests that entropy already provides a reliable, parameter-free starting point that a learned router can refine.

\noindent\textbf{Softer privileged information beyond ground-truth traces.}
The DASD mechanism does not depend on $r$ being a complete verified solution; it requires only that conditioning on $r$ collapses the teacher's residual uncertainty in a useful direction. This opens the door to weaker privileged signals that are cheaper to obtain at scale: partial hints, intermediate verifier feedback, retrieved solution fragments, and tool or environment observations in agentic settings. Each admits the same dual-forward and signed-routing structure as DASD; the only ingredient that changes is what is prepended to the privileged branch. We see this as a promising route to extending direction-adaptive self-distillation beyond regimes where verified ground-truth traces happen to be readily available, into domains where complete reference solutions do not exist.

\noindent\textbf{Direction-adaptive supervision as a general principle.}
The token-level pathology DASD addresses---a single supervision direction applied uniformly to a structurally heterogeneous population of tokens---is, plausibly, not exclusive to mathematical reasoning. Wherever a dense token-level objective is overlaid on top of a sparse outcome signal, the same dichotomy between forking decisions and routine execution can re-emerge: alignment data containing both content-shaping and stylistic tokens, code generation interleaving syntactic glue and algorithmic choices, and agentic post-training mixing planning steps with tool-use steps. We therefore suspect that the directional credit principle isolated here generalizes beyond the OPSD failure mode, and we view testing it on alignment, code, and agentic post-training as a natural way to delineate when direction matters and when uniform pressure is sufficient.

\end{document}